\documentclass{article} 
\usepackage[preprint]{neurips_2025}
\usepackage{gensymb}

\usepackage[utf8]{inputenc} 
\usepackage[T1]{fontenc}    
\usepackage{hyperref}       
\usepackage{url}            
\usepackage{booktabs}       
\usepackage{amsfonts}       
\usepackage{nicefrac}       
\usepackage{microtype}      
\usepackage{xcolor}         
\usepackage{wrapfig}
\usepackage{graphicx}

\usepackage{booktabs}
\usepackage{graphicx} 

\usepackage{subfigure}
\usepackage{bm}
\usepackage{hyperref}
\usepackage{url}
\usepackage{framed}
\usepackage{subcaption} 
\usepackage{wrapfig}
\usepackage{amsthm}
\usepackage{makecell}
\usepackage{array}
\usepackage{multirow}
\usepackage{amsfonts}
\usepackage{amssymb}
\usepackage{amsmath} 

\renewcommand{\eqref}[1]{\textup{Eq.~(\ref{#1})}}

\def \L{\mathcal{L}}

\def \z{\mathbf{z}}

\def \L{\mathcal{L}}

\def \x{\mathbf{x}}

\def \L{\mathcal{L}}

\usepackage{amsmath, amsthm}
\usepackage{graphicx}
\usepackage{booktabs}

\newtheorem{lemma}{Lemma}

\newcolumntype{C}[1]{>{\centering\arraybackslash}m{#1}}
\usepackage{booktabs}
\usepackage{multirow}
\usepackage{array}            

\title{Semi-Supervised Contrastive Learning with Orthonormal Prototypes}

\author{
Huanran Li, Manh Nguyen, \& Daniel Pimentel-Alarc\'on\\
Department of Electrical Engineering, Statistics, Biostatistics\\
Wisconsin Institute of Discovery\\
University of Wisconsin-Madison\\
\texttt{\{hli488\}\{mdnguyen4\}\{pimentelalar\}@wisc.edu}
}
\begin{document}

\maketitle

\begin{abstract}
Contrastive learning has emerged as a powerful method in deep learning, excelling at learning effective representations through contrasting samples from different distributions. However, dimensional collapse, where embeddings converge into a lower-dimensional space, poses a significant challenge, especially in semi-supervised and self-supervised setups. In this paper, we first identify a critical learning-rate threshold, beyond which standard contrastive losses converge to collapsed solutions. Building on these insights, we propose CLOP, a novel semi-supervised loss function designed to prevent dimensional collapse by promoting the formation of orthogonal linear subspaces among class embeddings. Through extensive experiments on real and synthetic datasets, we demonstrate that CLOP improves performance in image classification and object detection tasks while also exhibiting greater stability across different learning rates and batch sizes.
\end{abstract}

\section{Introduction}
Recent advancements in deep learning have positioned {Contrastive Learning} as a leading paradigm, largely due to its effectiveness in learning representations by contrasting samples from different distributions while aligning those from the same distribution. Prominent models in this domain include SimCLR \cite{chen2020simple}, Contrastive Multiview Coding (CMC) \cite{tian2020contrastive}, VICReg \cite{bardes2021vicreg}, BarLowTwins \cite{zbontar2021barlow}, among others \cite{henaff2020data, li2020prototypical, wu2018unsupervised}.
These models share a common two-stage framework: representation learning and fine-tuning. In the first stage, representation learning is performed in a self-supervised manner, where the model is trained to map inputs to embeddings using contrastive loss to separate samples from different labels. In the second stage, fine-tuning occurs under a supervised setup, where labeled data is used to classify embeddings correctly. For practical applicability, a small amount of labeled data is required in the fine-tuning stage to produce meaningful classifications, making the overall pipeline semi-supervised. 
Empirical evidence demonstrates that these models, even with limited labeled data (as low as 10\%), can achieve performance comparable to fully-supervised approaches on moderate to large datasets \cite{jaiswal2020survey}.

Despite the effectiveness of contrastive learning on largely unlabeled datasets, a common issue encountered during the training process is {dimensional collapse}. As pointed out by \cite{ fu2022details, gill2024engineering, hassanpour2024overcoming,jing2021understanding,rusak2022content, tao2024breaking, xue2023features}, this phenomenon describes the collapse of output embeddings from the neural network into a lower-dimensional space, reducing their spatial utility and leading to indistinguishable classes. There are two main approaches to resolve this issue: augmentation modification \cite{fu2022details, jing2021understanding, tao2024breaking, xue2023features} and loss modification \cite{fu2022details, hassanpour2024overcoming, rusak2022content}.
In this paper, we propose a semi-supervised loss function CLOP to address the issue of collapse. To address dimensional collapse, our approach selects prototypes similarly to \cite{gill2024engineering, zhu2022balanced}. The key distinction is that our method aims to push the embeddings toward distinct orthogonal linear subspaces, allowing them to occupy a higher-rank space. We demonstrate through experiments that CLOP is more effective for image classification and object detection tasks.

The main contributions of this paper can be viewed from two perspectives. First, we identify and analyze the detrimental effects of inappropriate learning rates on contrastive learning losses that rely exclusively on cosine similarity. Specifically, we show that when embeddings collapse into a rank-1 subspace, the optimizer converges to an uninformative stationary point. Through a gradient‐descent analysis, we establish the existence of a critical learning-rate range outside of which training inevitably converges to this undesirable point.
Building upon these insights, we propose a novel contrastive loss term, termed CLOP, designed to encourage embeddings of partial training datasets to cluster around a set of orthonormal prototypes. This loss term is applicable to both semi-supervised and fully-supervised contrastive learning scenarios, particularly when only a subset of the training data is labeled.
We validate CLOP’s effectiveness through extensive experiments on embedding visualization, image classification, and object detection tasks, using both balanced and imbalanced datasets. Our empirical results demonstrate that CLOP consistently outperforms baseline methods, exhibiting remarkable robustness and stability across a wide range of learning rates and reduced batch sizes.

\textbf{Paper Organization}
In Section \ref{sec-relatedwork}, we provide essential background information and discuss recent advancements in both self-supervised and supervised contrastive learning, as well as analyzing the dimensional collapse phenomenon associated with contrastive learning methods. Section \ref{sec-theory} elaborates on the motivation behind our approach through simulations and detailed gradient analysis. Subsequently, in Section \ref{sec-model}, we introduce our proposed model, \textbf{CLOP}. Finally, Section \ref{sec-experiment} presents extensive experimental evaluations conducted on image datasets.

\section{Related Work}
\label{sec-relatedwork}
Contrastive learning has gained prominence in deep learning for its ability to learn meaningful representations by pulling together similar (positive) pairs and pushing apart dissimilar (negative) pairs in the embedding space. Positive pairs are generated through techniques like data augmentation, while negative pairs come from unrelated samples, making contrastive learning particularly effective in self-supervised tasks like image classification. Pioneering models such as SimCLR \cite{chen2020simple}, CMC \cite{tian2020contrastive}, VICReg \cite{bardes2021vicreg}, and Barlow Twins \cite{zbontar2021barlow} share the objective of minimizing distances between augmented versions of the same input (positive pairs) and maximizing distances between unrelated inputs (negative pairs). SimCLR maximizes agreement between augmentations using contrastive loss, while CMC extends this to multi-view learning \cite{chen2020simple, tian2020contrastive}. VICReg introduces variance-invariance-covariance regularization without relying on negative samples \cite{bardes2021vicreg}, and Barlow Twins reduce redundancy between different augmentations \cite{zbontar2021barlow}.

Recent innovations have improved contrastive learning across various domains. For instance, methods like structure-preserving quality enhancement in CBCT images \cite{kang2023structure} and false negative cancellation \cite{huynh2022boosting} have enhanced image quality and classification accuracy. In video representation, cross-video cycle-consistency and inter-intra contrastive frameworks \cite{wu2021contrastive, tao2022improved} have shown significant gains. Additionally, contrastive learning has advanced sentiment analysis \cite{xu2023improving}, recommendation systems \cite{yang2022supervised}, and molecular learning with faulty negative mitigation \cite{wang2022improving}. \cite{xiao2024simple} introduces GraphACL, a novel framework for contrastive learning on graphs that captures both homophilic and heterophilic structures without relying on augmentations. 


\subsection{Contrastive Loss}
In unsupervised learning, \cite{wu2018unsupervised} introduced InfoNCE, a loss function defined as:
\begin{equation}
    \mathcal{L}_{\text{infoNCE}} = - \sum_{i \in I} \log \frac{\exp(\mathbf{z}_i^\top \mathbf{z}_{j(i)}/\tau)}{\sum_{a \neq i} \exp(\mathbf{z}_i^\top \mathbf{z}_a/\tau)} 
    \label{eq-InfoNCE_loss}
\end{equation}
where $\mathbf{z}_i$ is the embedding of sample $i$, $j(i)$ its positive pair, and $\tau$ controls the temperature.

Recent refinements focus on (1) component modifications, (2) similarity adjustments, and (3) novel approaches. \cite{li2020prototypical} use EM with k-means to update centroids and reduce mutual information loss, while \cite{wang2022rethinking} add L2 distance to InfoNCE, though both underperform state-of-the-art (SOTA) techniques. \cite{xiao2020should} reduce noise with augmentations, and \cite{yeh2022decoupled} improve gradient efficiency with Decoupled Contrastive Learning, though neither surpasses SOTA.
In similarity adjustments, \cite{chuang2020debiased} propose a debiased loss, and \cite{ge2023hyperbolic} use hyperbolic embeddings, but neither outperforms SOTA. Novel methods include min-max InfoNCE \cite{tian2020makes}, Euclidean-based losses \cite{bardes2021vicreg}, and dimension-wise cosine similarity \cite{zbontar2021barlow}, achieving competitive performance without softmax-crossentropy.

\subsection{Semi-Supervised Contrastive Learning}

Semi-supervised contrastive learning effectively leverages both labeled and unlabeled data to learn meaningful representations. \cite{zhang2022semi} introduced a framework with similarity co-calibration to mitigate noisy labels by adjusting the similarity between pairs. \cite{inoue2020semi} proposed Generalized Contrastive Loss (GCL), which unifies supervised and unsupervised learning for speaker recognition, while \cite{kim2021selfmatch} combined contrastive self-supervision with consistency regularization in SelfMatch.
\cite{sohn2020fixmatch} introduced FixMatch, which combines pseudo-labeling with consistency regularization. In this approach, weakly augmented samples generate pseudo-labels that guide strongly augmented versions, ensuring robust semi-supervised learning (SSL) performance. \cite{yang2022class} enhanced SSCL by enforcing class-wise consistency in learned representations, improving robustness to class imbalance and increasing generalization across datasets.
\cite{zheng2022simmatch} proposed SimMatch, a framework that unifies contrastive learning and consistency regularization by optimizing both instance-level alignment and class-level semantic consistency, leading to improved SSL feature representations. Building on this, \cite{zheng2023simmatchv2} introduced SimMatch-V2, refining the balance between contrastive and consistency learning objectives, further enhancing transferability and performance in semi-supervised settings.

\subsection{Dimensional Collapse}
For dimensional collapse in contrastive learning, \cite{jing2021understanding} examine dimensional collapse in self-supervised learning. They attribute this to strong augmentations distorting features and implicit regularization driving weights toward low-rank solutions. 
Similarly, \cite{xue2023features} explore how simplicity bias leads to class collapse and feature suppression, with models favoring simpler patterns over complex ones. They suggest increasing embedding dimensionality and designing augmentation techniques that preserve class-relevant features to counter this bias and promote diverse feature learning. \cite{fu2022details} emphasize the role of data augmentation and loss design in preventing class collapse, proposing a class-conditional InfoNCE loss term that uniformly pulls apart individual points within the same class to enhance class separation.
In supervised contrastive learning, \cite{gill2024engineering} propose loss function modifications to follow an ETF geometry by selecting prototypes that form this structure. In graph contrastive learning, \cite{tao2024breaking} introduce a whitening transformation to decorrelate feature dimensions, avoiding collapse and enhancing representation capacity. 
Finally, \cite{rusak2022content} investigate the preference of contrastive learning for content over style features, leading to collapse. They propose to leverage adaptive temperature factors in the loss function to improve feature representation quality.

\section{Motivation}
\label{sec-theory}
This section investigates the cause of {dimensional collapse} in the perspective of learning rates. We first demonstrate that {dimensional collapse} represents a local stationary point of the InfoNCE loss by showing that linear embeddings result in a zero gradient (Lemma~\ref{subspace-embeddings}). While our demonstration utilizes the InfoNCE loss, this conclusion generalizes to most current loss functions relying solely on cosine similarity as their metric, encompassing unsupervised \cite{henaff2020data, chen2020simple, cui2021parametric, xiao2020should, yeh2022decoupled, wang2022rethinking, li2024transfusion}, semi-supervised \cite{hu2021semi, shen2021semi}, and supervised contrastive learning \cite{khosla2020supervised, cui2021parametric, peeters2022supervised, li2022targeted}. Subsequently, we discuss how a large learning rate induces a shift in the embedding mean through comparisons between pairs of negative samples, ultimately leading to {dimensional collapse}. 


\label{sec:complete_collapse}

The InfoNCE loss (\eqref{eq-InfoNCE_loss}) aims to encourage the embeddings to form distinguishable clusters in high-dimensional space, thereby facilitating classification for downstream models. {However, in Lemma~\ref{subspace-embeddings}, we demonstrate that the worst-case scenario — where all embeddings become identical or co-linear — also constitutes a local optimum for the InfoNCE loss. This observation suggests that, from a theoretical perspective, InfoNCE exhibits instability, as both the best and worst solutions can lead to stationary points.}
\begin{lemma}
    \label{subspace-embeddings}
   Let \(\mathcal{F}: \mathbb{R}^{m} \to \mathbb{R}^{m'}\) be a family of Contrastive Learning structures, where \(m\) and \(m'\) denote the dimensions of the inputs and embeddings, respectively. If a function \(f \in \mathcal{F}\) is trained using the InfoNCE loss, then there exist infinitely many local stationary points where all embeddings produced by \(f\) are \underline{all equal} or \underline{co-linear}.
\end{lemma}

The proof of Lemma~\ref{subspace-embeddings} relies on the observation that the embeddings are only compared against each other. If all embeddings are either identical or co-linear, the gradient vanishes due to the lack of angular differences, as well as the normalization process. 
The full proof is presented in Appendix~\ref{proof1}.

\begin{figure}
\setlength{\tabcolsep}{0pt} 
  \centering
  \begin{tabular}{rrr}
    \includegraphics[height=0.28\linewidth, trim={0 0 2.55cm 0}, clip]{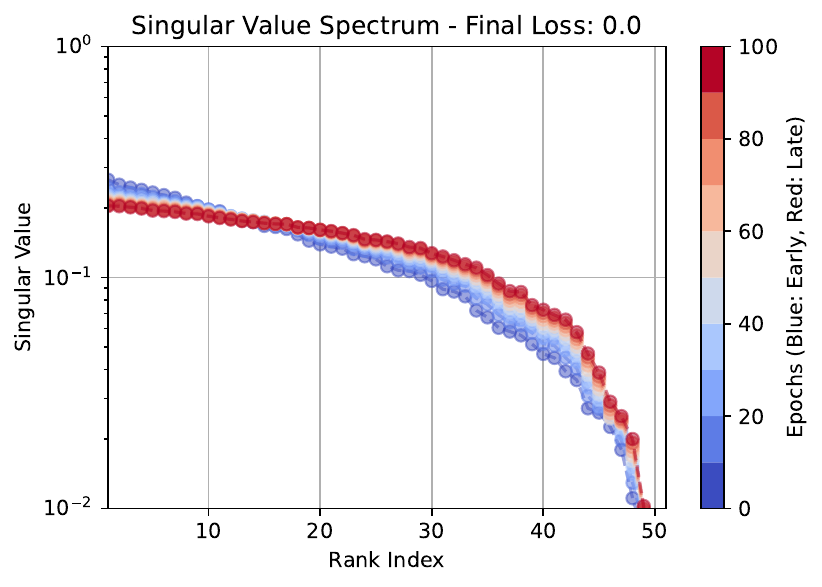} &
    \includegraphics[height=0.28\linewidth, trim={0.6cm 0 2.55cm 0}, clip]{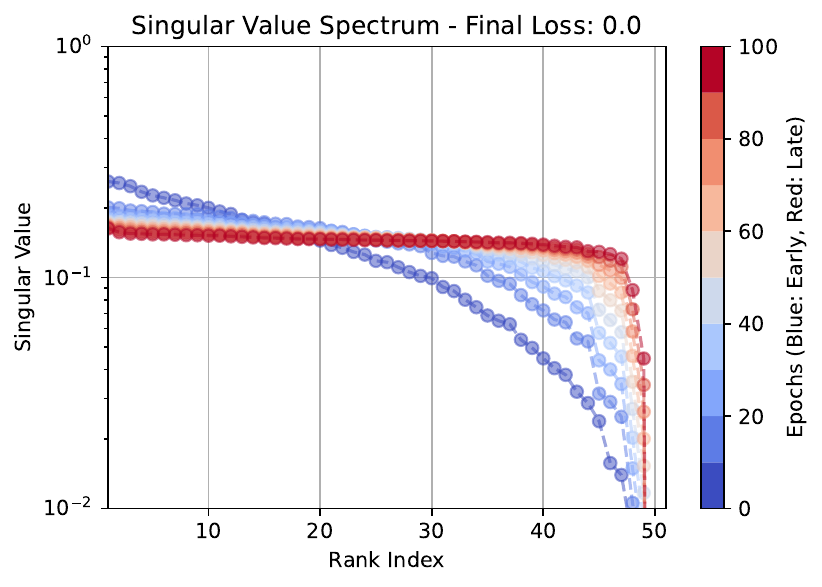} &
    \includegraphics[height=0.28\linewidth, trim={0.6cm 0 2.55cm 0}, clip]{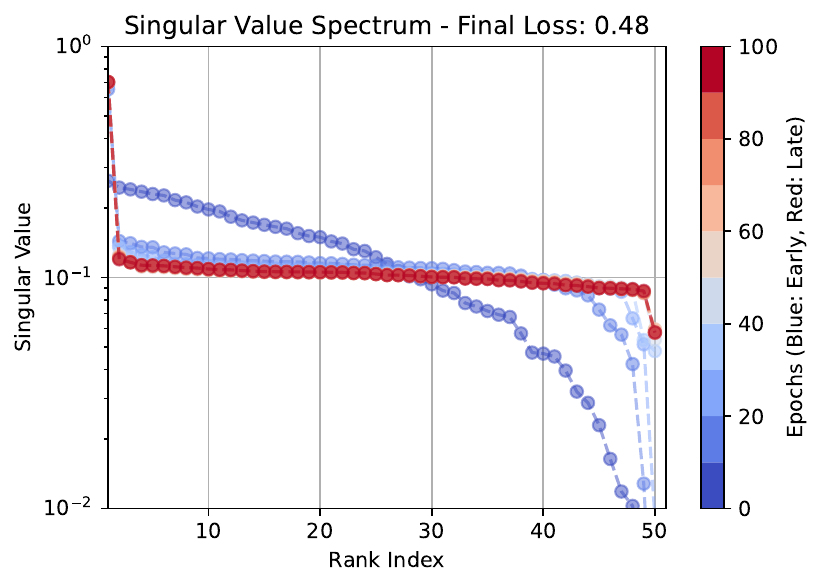} 
    \includegraphics[height=0.28\linewidth, trim={11.63cm 0 0 0}, clip]{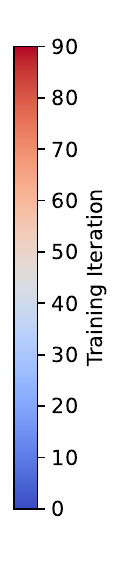}\\[-0.1cm]
    \includegraphics[height=0.28\linewidth, trim={0cm 0 2.32cm 0}, clip]{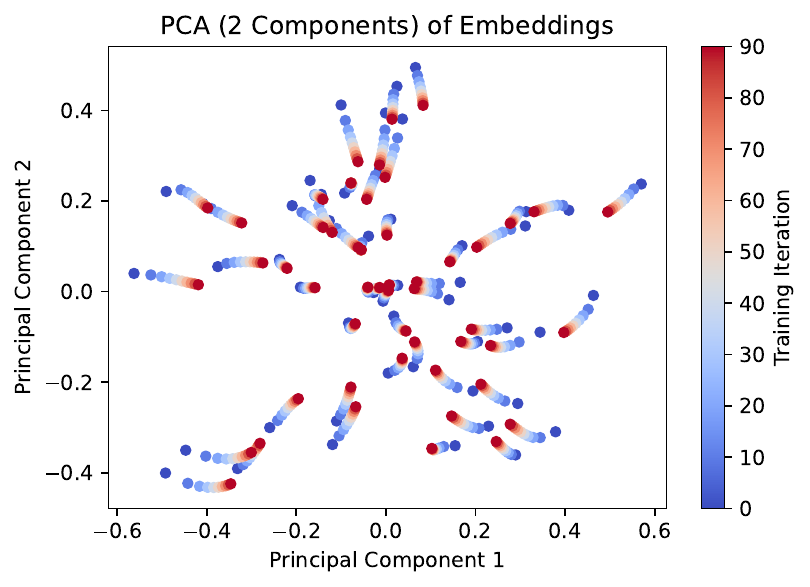} &
    \includegraphics[height=0.28\linewidth, trim={0.6cm 0 2.32cm 0}, clip]{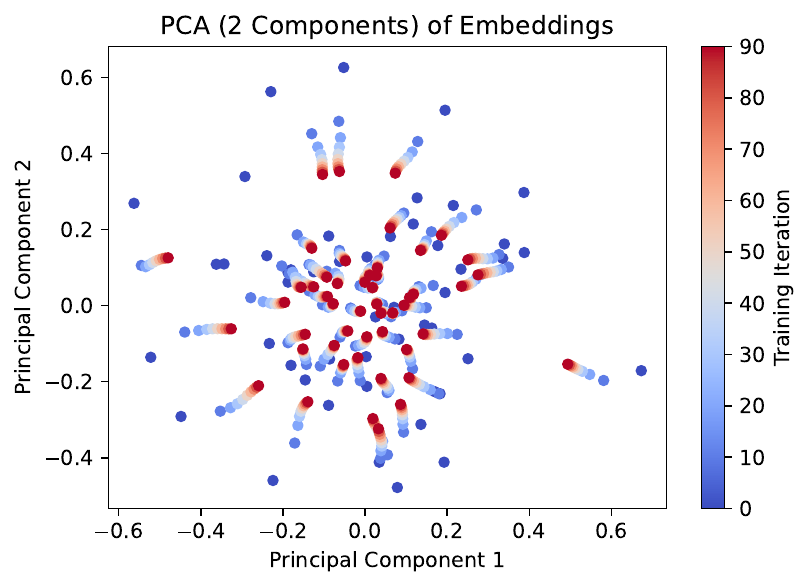} &
    \includegraphics[height=0.28\linewidth, trim={0.6cm 0 0 0}, clip]{Figures/LR_Collapse_Sim/PCA_figure_lr1000.pdf} \\
    \multicolumn{1}{c}{{\small \textbf{(a)} Learning rate of 0.01}} &
    \multicolumn{1}{c}{{\small \textbf{(b)} Learning rate of 0.1}} &
    \multicolumn{1}{c}{{\small \textbf{(c)} Learning rate of 1}}
  \end{tabular}
  \caption{Simulation with Repulsive Force on 50 simulated points in 50-dimensional space.}
  \label{fig:embeddings-lr}
\end{figure}

 The remainder of this section examines the role of a large learning rate in causing {dimensional collapse}. 
To understand the dynamics of contrastive learning, it is crucial to consider two forces acting on each embedding: the \textit{gravitational force} within the same pseudo class and the \textit{repulsive force} between different pseudo classes. Contrary to the common belief that the gravitational force is responsible for inducing collapse, we observe that the overshooting of the repulsive force could be directly related to the {dimensional collapse} in contrastive learning.

To better illustrate the repulsive force, we conducted a simulation using 50 randomly generated embeddings in 50-d space, where positive pairs initially coincide at a single point, reflecting a scenario where the model has successfully merged augmented variants from the same input source into one embedding. We subsequently performed gradient descent on these embeddings using the InfoNCE loss with temperature of 0.1, recording the embedding trajectories throughout training. The simulations were executed across three distinct learning rates: 0.01, 0.1, and 1. The embedding singular value spectrum and the first two principal components of these embeddings are visualized in Figure~\ref{fig:embeddings-lr}. At learning rates of 0.01 and 0.1, the repulsive force inherent to the InfoNCE loss effectively redistributes the embeddings across a more uniform space, as indicated by the more evenly dispersed singular values across the embedding dimensions. Conversely, at a learning rate of 1, the embeddings fail to redistribute and instead collapse into a one-dimensional subspace, corroborating the stationary point described by Lemma~\ref{subspace-embeddings}. Furthermore, analysis of the first two principal components clearly illustrates a significant shift in the embedding mean and a reduction in variance throughout the training iterations. This observation indicates that larger learning rates induce a substantial shift in the embedding mean, consequently accelerating the collapse process by aligning gradients in similar directions.
Thus, establishing an upper bound on the learning rates is crucial, as it limits the shift in embedding mean and would contribute to preventing dimensional collapse.

To better understand this phenomenon, we conducted a theoretical analysis of the gradient descent process using the InfoNCE loss, detailed in Appendix~\ref{proof2}. Briefly, our analysis reveals that, after performing a single gradient descent step, the upper bound on the norm of the embedding mean is scaled by the factor:
$$\|\bm{\mu}_1\|_2 \leq \left| \left(1 - \frac{2\eta}{\tau}\right)\frac{2}{\sigma^2} \frac{{{|\mathcal{N}|}}}{1 + |\mathcal{N}| }\right| \|\bm{\mu}_0\|_2,$$
where $\sigma$ represents the minimum embedding norm before normalization, and $|\mathcal{N}|$ is the number of negative samples within the batch. To prevent an increase in the embedding mean's norm, we constrain this scaling factor to be less than or equal to 1. This constraint yields an optimal learning rate range given by:
$\eta \in \left[\frac{\tau}{2} \left(1 - \frac{\sigma^2 (1 + |\mathcal{N}|)}{2|\mathcal{N}|}\right), \frac{\tau}{2} \left(1 + \frac{\sigma^2 (1 + |\mathcal{N}|)}{2|\mathcal{N}|} \right)\right].$ 
This suggests that, under the assumption that positive sample pairs are successfully merged into the same embedding, the optimal learning rate that reduce mean shift is $\frac{\tau}{2}$.

\section{CLOP: Populating Embedding Rank with Orthonormal Prototypes}
\label{sec-model}
To avoid the issue of embedding collapsing into a rank-1 linear subspace, we introduce a novel approach that promotes point isolation by adding an additional term to the loss function for contrastive learning. Specifically, we initialize a group of \textit{orthonormal prototypes}. The number of {orthonormal prototypes} matches the total number of classes in the dataset. We then maximize the similarity between the {orthonormal prototypes} and the labeled samples in the training set. 

Formally, let $\mathcal{S}$ be the labeled training set containing pairs of embeddings and labels, denoted as $\mathcal{S} = \{(\mathbf{z}_i, y_i) \mid i \in \{1, \dots, |\mathcal{S}|\}\}$. The set of prototypes, denoted as $\mathcal{C}$, is defined as $\mathcal{C} = \{\mathbf{c}_1, \dots, \mathbf{c}_k\}$, where $k$ represents the number of classes in the dataset. To generate the prototypes $\mathcal{C}$, we randomly sample $k$ i.i.d. vectors from an $m'$-dimensional space, where $|\mathbf{z}_i| = m'$. Subsequently, we apply singular value decomposition (SVD) to obtain the orthonormal basis, denoted as $\mathcal{C}$. This ensures that each prototype $\mathbf{c}_i$ is initialized as a unit vector, orthogonal to all other prototypes, at the beginning of the training process. The {CLOP} loss is formulated as follows:
\begin{align}
    \mathcal{L}_{\text{CLOP}} = \mathcal{L}_{\text{infoNCE}} + \lambda \frac{1}{|\mathcal{S}|} \sum_{i=1}^{|\mathcal{S}|} (1 - s(\mathbf{z}_i, \mathbf{c}_{y_i})),
    \label{CLOP_loss}
\end{align}
where $s(\cdot, \cdot)$  denotes the similarity metric, typically chosen to be the same as that used in $\mathcal{L}_{\text{infoNCE}}$, namely cosine similarity.

The primary objective of the {CLOP} loss is to align all embeddings corresponding to the same class towards a common target prototype, $\mathbf{c}_{y_i}$. Beyond the ``gravitational force" and ``repulsive force" provided by the main contrastive loss, the {CLOP} loss introduces a supervised ``pulling force" that prevents collapse by pulling labeled embeddings into class-specific orthogonal subspaces.
It is important to note that, without additional constraints, samples outside of set $\mathcal{S}$ may still converge to other unspecified embeddings, potentially collapsing into a rank-1 subspace. However, a fundamental assumption in contrastive learning is that augmented samples are treated as being drawn from the same distribution as the original input data from the same class. Thus, the ``gravitational force'' between embeddings of the same class should pull unsupervised embeddings toward the target prototypes.

To assess the supervisory efficacy of CLOP, we visualized the CIFAR100's representations learned by both CLOP and SupCon under two label regimes (10\% and 100\%) in Table~\ref{tab:fivefigs}. Embeddings were projected to two dimensions using t-distributed Stochastic Neighbor Embedding (t-SNE) \cite{van2008visualizing}, and we also examined each model’s singular value spectrum. Both methods employ a ResNet-50 backbone with linear projection heads four times wider than standard (denoted as \textit{ResNet-50 (4x)}), trained for 500 epochs with a batch size of 256 and an initial learning rate of 1.0.
CLOP yields a higher effective rank in its embedding space and produces more compact, well-separated clusters under both limited (10\%) and full (100\%) supervision. Quantitatively, linear fine-tuning on ImageNet confirms this advantage: at 10\% label usage, CLOP achieves a top-1 accuracy of 59.57\%, outperforming SupCon’s 55.01\% by 4.56\%; at 100\% label usage, CLOP reaches 82.06\% versus SupCon’s 74.75\%, a gain of 7.31\%. This experiment is conducted on a single NVIDIA A100 GPU, with each run completing within 4 hours of training.

\begin{table}
  \centering
  \setlength{\tabcolsep}{3pt}
  \begin{tabular}{
     c| c c c
    }
    \toprule
      & Singular Value Spectrum
      & SupCon
      & CLOP \\
    \midrule
    \multirow{2}{*}[12ex]{\rotatebox{90}{10\% Label}}
      & \includegraphics[height=3.4cm,trim={0.7cm 0.7cm 0cm 0.66cm},clip]{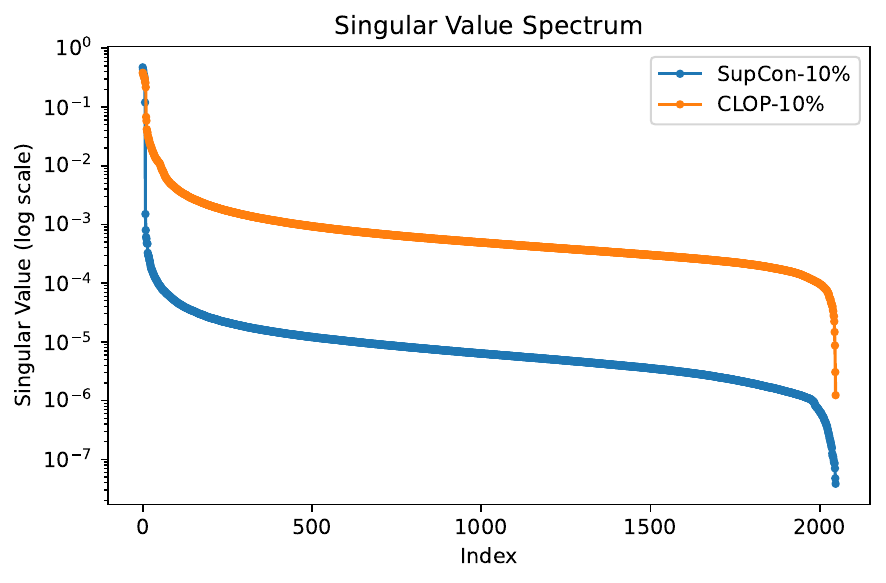}
      & \includegraphics[height=3.4cm,trim={0.7cm 0.7cm 2cm 0.66cm},clip]{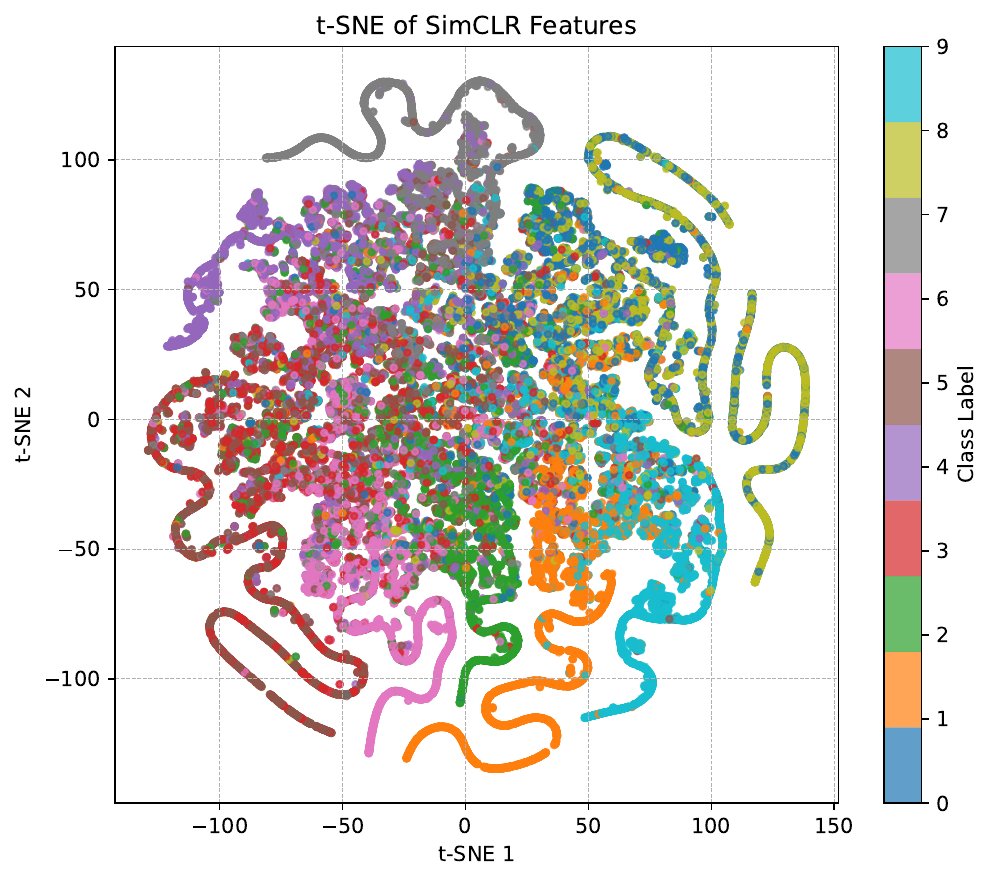}
      & \includegraphics[height=3.4cm,trim={0.7cm 0.7cm 2cm 0.66cm},clip]{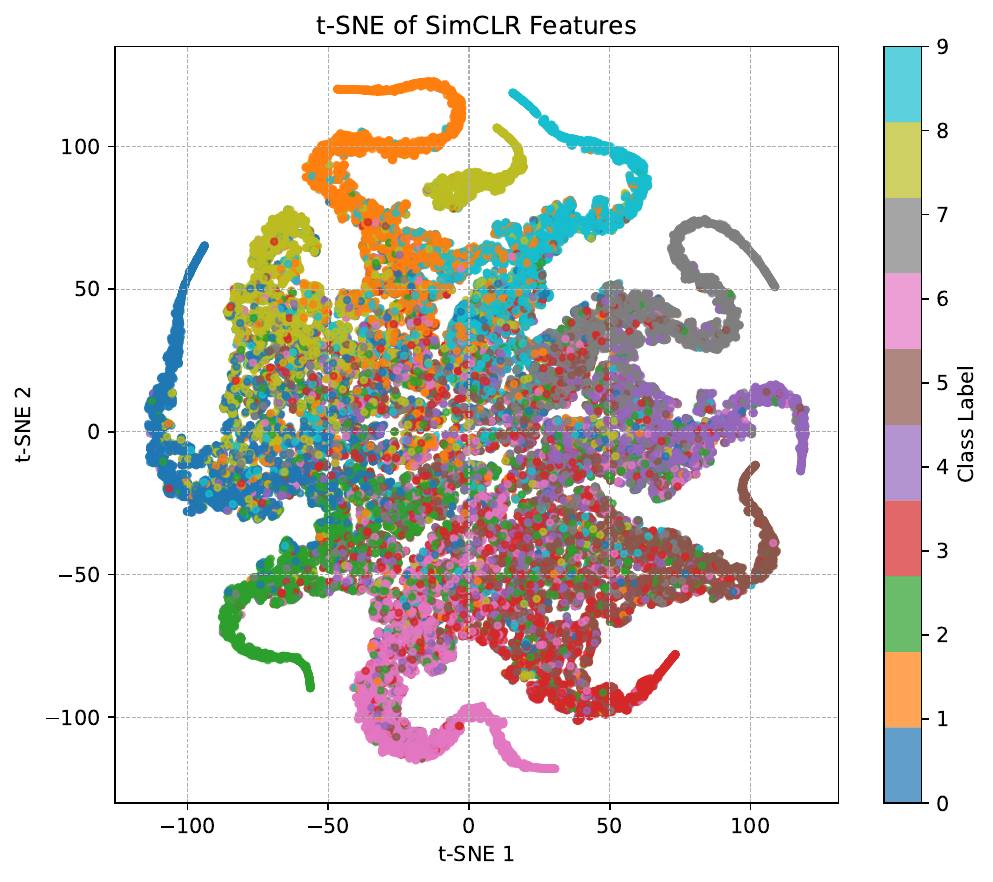} \\
      &  
      & Top-1 Acc: 55.01\%
      & Top-1 Acc: 59.57\% \\
    \midrule
    \multirow{2}{*}[12ex]{\rotatebox{90}{100\% Label}}
      & \includegraphics[height=3.4cm,trim={0.7cm 0.7cm 0 0.66cm},clip]{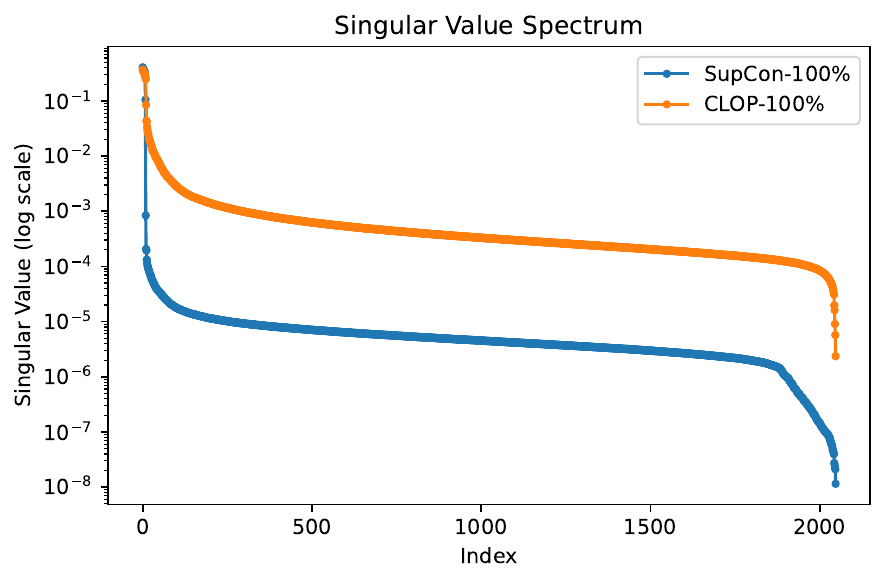}
      & \includegraphics[height=3.4cm,trim={0.7cm 0.7cm 2cm 0.66cm},clip]{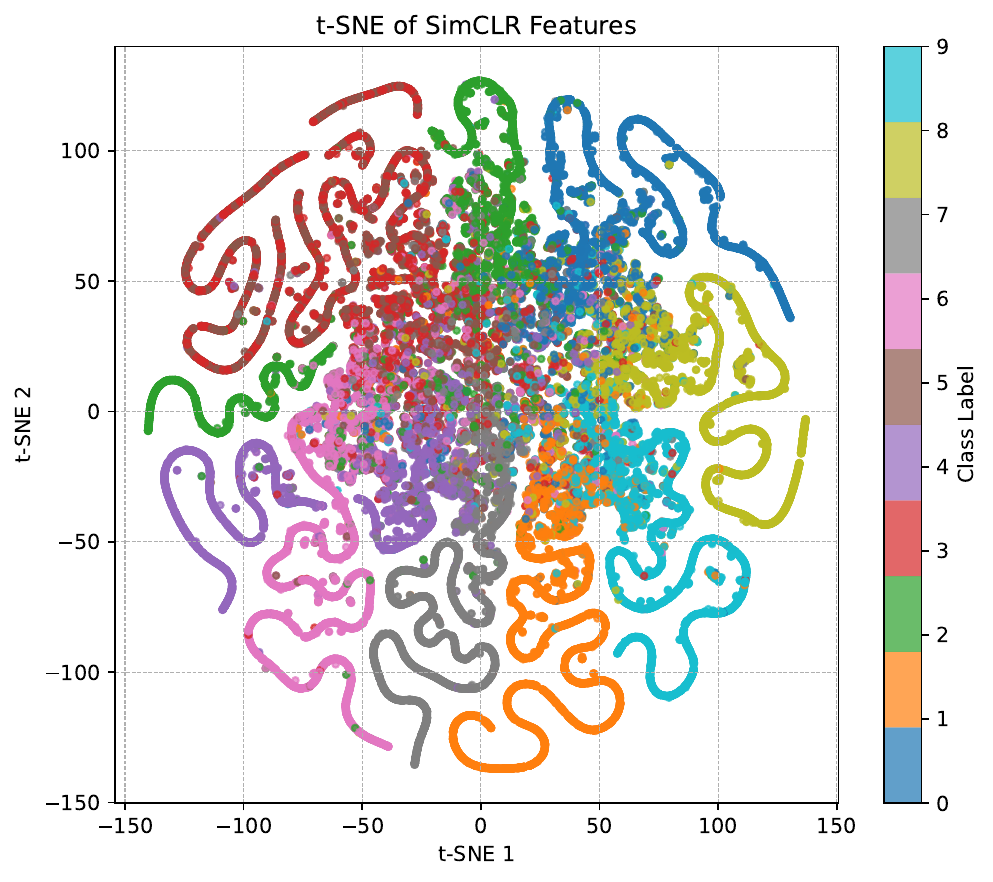}
      & \includegraphics[height=3.4cm,trim={0.7cm 0.7cm 2cm 0.66cm},clip]{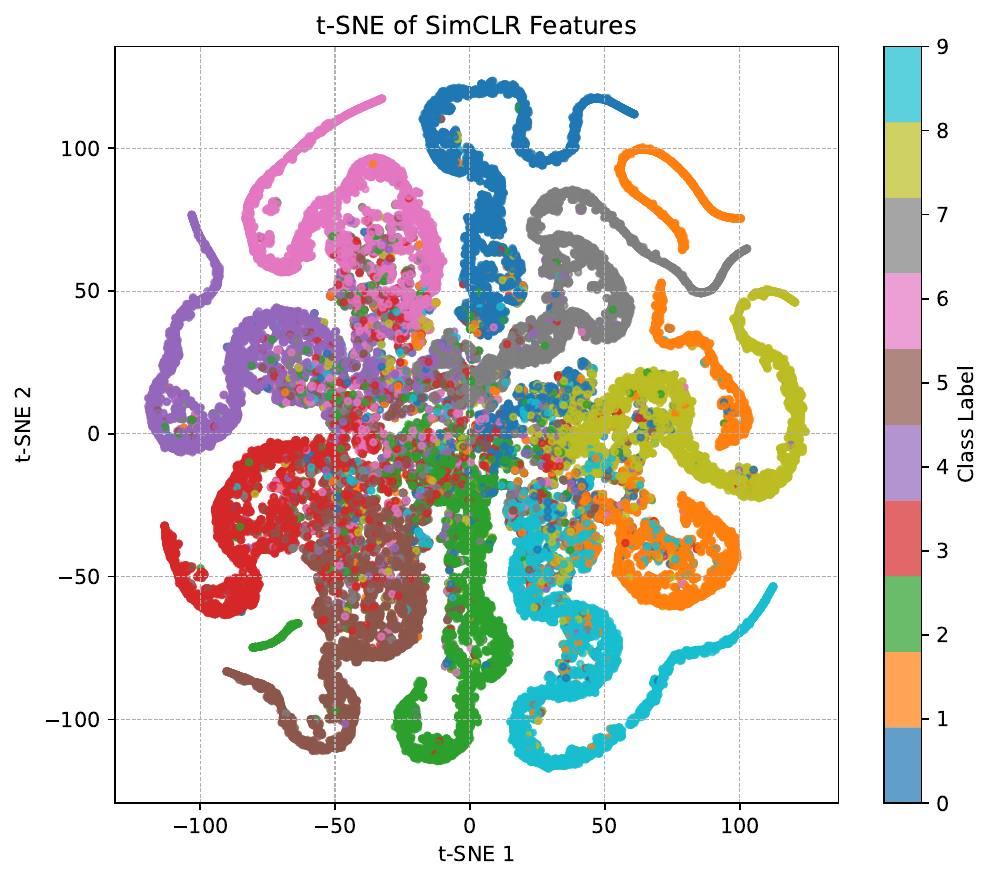} \\
      &  
      & Top-1 Acc: 74.75\%
      & Top-1 Acc: 82.06\% \\
    \bottomrule
  \end{tabular}
  \vspace{0.1cm}
  \caption{SupCon vs.\ CLOP under different label usage}
  \label{tab:fivefigs}
\end{table}

\section{Experiment}
\label{sec-experiment}

\begin{table}
    \centering
    \setlength{\tabcolsep}{4pt}
    \begin{minipage}[t]{0.48\textwidth}
    \begin{tabular}{lcccc}
        \toprule
        Methods 
        & \multicolumn{2}{c}{10\% Labels} & \multicolumn{2}{c}{50\% Labels} \\
        \cmidrule(lr){2-3} \cmidrule(lr){4-5}
        & Top-1 & Top-5 & Top-1 & Top-5 \\
        \midrule
        SupCon       &0.595 &0.865 &0.625 &0.884  \\
        FixMatch     &0.588 &0.883 &0.611 &0.887 \\
        SsCL         &0.715 &0.829 &0.746 &0.887 \\
        CCSSL        &0.735 &0.841 &0.757 &0.894 \\
        SimMatch     &0.719 &0.829 &0.724 &0.866 \\
        SimMatch-V2  &0.729 &0.840 &0.729 &0.877 \\
        \bottomrule
        CLOP       &\textbf{0.743} &\textbf{0.904} &\textbf{0.76} &\textbf{0.92} \\
        \bottomrule
    \end{tabular}
    \caption{Top-1 and Top-5 accuracy on CIFAR-100.}
    \label{tab:CIFAR-classification}
    \end{minipage}
    \hfill
    \begin{minipage}[t]{0.48\textwidth}
    \centering
    \begin{tabular}{lcccc}
        \toprule
        Methods
        & \multicolumn{2}{c}{10\% Labels} & \multicolumn{2}{c}{50\% Labels} \\
        \cmidrule(lr){2-3} \cmidrule(lr){4-5}
        & Top-1 & Top-5 & Top-1 & Top-5 \\
        \midrule
        SupCon       &0.704 &0.874 &0.734 &0.830 \\
        FixMatch     &0.720 &0.886 &0.774 &0.911 \\
        SsCL         &0.721 &0.909 &0.786 &0.899 \\
        CCSSL        &0.751 &0.923 &0.771 &0.907 \\
        SimMatch     &0.740 &\textbf{0.930} &0.776 &0.904 \\
        SimMatch-V2  &0.748 &0.917 &0.788 &0.916 \\
        \bottomrule
        CLOP       &\textbf{0.791} &{0.927} &\textbf{0.829} &\textbf{0.949} \\
        \bottomrule
    \end{tabular}
    \caption{Top-1 and Top-5 accuracy on ImageNet.}
    \label{tab:ImageNet-classification}
    \end{minipage}
\end{table}

In this section, we evaluate the performance of CLOP in various learning settings. First, we compare CLOP against existing semi-supervised contrastive learning methods on image classification tasks. Furthermore, we show that CLOP consistently outperforms competitors under both balanced and imbalanced class distributions. 
Next, we highlight the generalizability of CLOP by presenting its outstanding transfer learning results on image classification and object detection tasks. Finally, we conduct extensive ablation studies on key hyperparameters, including learning rate, batch size, $\lambda$, similarity metrics, and augmentation strategies, highlighting CLOP’s robustness to varying learning rates and its effectiveness in small-batch learning. 

\begin{table}
\centering
\label{tab:accuracy_comparison}
    \begin{tabular}{lcccccccc}
    \toprule
    {Methods} & \multicolumn{2}{c}{CIFAR-100} & \multicolumn{2}{c}{ImageNet-200} & \multicolumn{2}{c}{ImageNet} \\
    \cmidrule(lr){2-3} \cmidrule(lr){4-5} \cmidrule(lr){6-7}
    & Top-1 & Top-5 & Top-1 & Top-5 & Top-1 & Top-5 \\
    \midrule
        SupCon &0.585 &0.858 &0.505 &0.749 &0.672 &0.780\\
        FixMatch &0.576 &0.840 &0.621 &0.759 &0.708 &0.860\\
        SsCL &0.719 &0.860 &0.585 &0.733 &0.730 &0.862\\
        CCSSL &0.744 &0.874 &0.631 &0.889 &0.710 &0.865\\
        SimMatch &0.685 &0.829 &0.527 &0.768 &0.705 &0.855\\
        SimMatchV2 &0.689 &0.862 &0.659 &0.846 &0.733 &0.845\\
        \bottomrule
        CLOP &\textbf{0.763} &\textbf{0.918} &\textbf{0.689} &\textbf{0.898} &\textbf{0.799} &\textbf{0.932}\\
    \bottomrule
    \end{tabular}
    \vspace{0.1cm}
\caption{Top-1 and Top-5 accuracy on CIFAR-100, ImageNet-200, and ImageNet under imbalanced-class training.}
\label{tab:Imbalanced}
\end{table}

\begin{table*}
\setlength{\tabcolsep}{4pt} 
    \centering
    \begin{tabular}{lcccccccc}
        \toprule
        Method & Food & CIFAR10 & CIFAR100 & SUN397 & DTD & Caltech-101 & Flowers \\
        \midrule
        SimCLR & 0.8820 & 0.9770 & 0.8590 & \textbf{0.6350} & 0.7320 & 0.9210 & 0.9700 \\
        SupCon & 0.8723 & 0.9742 & 0.8427 & 0.5804 & 0.7460 & 0.9104 & 0.9600 \\
        FixMatch & 0.8824 & 0.9639 & 0.8553 & 0.5774 & 0.7269 & 0.9123 & 0.9669 \\
        SsCL & 0.8546 & 0.9866 & 0.8481 & 0.5800 & 0.7300 & 0.9115 & 0.9574 \\
        CCSSL & 0.8663 & 0.9637 & 0.8352 & 0.5818 & 0.7270 & 0.9029 & 0.9456 \\
        SimMatch & \textbf{0.8881} & 0.9759 & 0.8435 & 0.6004 & 0.7306 & 0.9013 & 0.9646 \\
        SimMatch-V2 & 0.8568 & 0.9638 & 0.8270 & 0.5886 & \textbf{0.7526} & 0.9185 & 0.9613 \\
        \bottomrule
        CLOP (this paper)  & 0.8792 & \textbf{0.9989} & \textbf{0.8809} & 0.6267 & 0.7385 & \textbf{0.9331} & \textbf{0.9718} \\
        \bottomrule
    \end{tabular}
    \caption{Transfer learning results for classification tasks (pretrained on ImageNet). 
    Numbers are mean-per-class accuracy for Caltech and Flowers; 
    and top-1 accuracy for all other datasets.}
    \label{tab:classification_performance}
\end{table*}

\subsection{Semi-Supervised Image Classification}
For balanced-class training, we utilize the full CIFAR-100~\cite{krizhevsky2009learning} and ImageNet~\cite{deng2009imagenet} datasets, considering scenarios where either $10\%$ or $50\%$ of labels are available for contrastive learning. CLOP is implemented with a ResNet-50 (4x) backbone using the SimCLR loss function. 
We benchmark CLOP against several state-of-the-art semi-supervised contrastive learning methods, including SupCon~\cite{khosla2020supervised}, FixMatch~\cite{sohn2020fixmatch}, SsCL~\cite{zhang2022semi}, CCSSL~\cite{yang2022class}, SimMatch~\cite{zheng2022simmatch}, and SimMatch-V2~\cite{zheng2023simmatchv2}. The classification results for CIFAR-100 and ImageNet are presented in Table~\ref{tab:CIFAR-classification} and Table~\ref{tab:ImageNet-classification}, respectively. Our findings indicate that CLOP consistently outperforms all competing methods across all experimental settings. Notably, on CIFAR-100, CLOP achieves the highest Top-1 accuracy of 0.743 with only 10\% of labels and maintains a strong lead at 0.760 with 50\% labels, while also significantly improving Top-5 performance. On ImageNet, CLOP achieved a Top-1 accuracy of 0.791 at 10\% label availability and 0.829 at 50\%, outperforming the next best methods by margins of over 4\% in some cases. 

For imbalanced-class training, we generate class-wise sample ratios by sampling from a uniform distribution in the range of $[0.001, 1]$, while maintaining all other experimental settings identical to the balanced-class training setup. The results are presented in Table~\ref{tab:Imbalanced}, demonstrating that CLOP significantly outperforms all benchmark methods.

\begin{table}
    \centering
    \begin{tabular}{lccccc}
        \toprule
        Method & Birdsnap & Cars & Aircraft & VOC2007 & Pets \\
        \midrule
        SimCLR & 0.7590 & 0.9130 & 0.8785 & 0.8410 & 0.8920 \\
        SupCon & 0.7515 & 0.9169 & 0.8409 & 0.8517 & 0.9347 \\
        FixMatch & 0.7545 & 0.9004 & 0.8462 & 0.8372 & \textbf{0.9515} \\
        SsCL & 0.7343 & 0.9089 & 0.8341 & 0.8504 & 0.9288 \\
        CCSSL & 0.7613 & \textbf{0.9247} & 0.8528 & 0.8580 & 0.9471 \\
        SimMatch & 0.7562 & 0.9127 & 0.8269 & 0.8393 & 0.9200 \\
        SimMatchV2 & 0.7516 & 0.9186 & 0.8453 & 0.8546 & 0.9494 \\
        \bottomrule
        CLOP & \textbf{0.7794} & 0.9171 & \textbf{0.8810} & \textbf{0.8646} & 0.8982 \\
        \bottomrule
    \end{tabular}
    \vspace{0.1cm}
    \caption{Transfer learning results for objects detection task (pretrained on ImageNet). 
    Numbers are mAP for VOC2007; mean-per-class accuracy for Aircraft and Pets; 
    and top-1 accuracy for Birdsnap.}
    \label{tab:object_detection_performance}
\end{table}

\subsection{Transfer Learning on Image Classification and Object Detection}
To assess CLOP's generalization capability on unseen datasets, we conduct transfer learning experiments on both image classification and object detection tasks. Specifically, we first pretrain the ResNet-50 (4x) backbone using various contrastive loss functions on ImageNet with all labels. Subsequently, we replace the projection head with either a one-layer prediction head or a two-layer object detection head with ReLU activation and fine-tune the network on the nature image datasets, including Food~\cite{bossard2014food}, CIFAR-10~\cite{krizhevsky2010convolutional}, CIFAR-100~\cite{krizhevsky2009learning}, SUN397~\cite{xiao2010sun}, DTD~\cite{qu2023towards}, Caltech-101~\cite{fei2004learning}, Flowers~\cite{nilsback2008automated}, Birdsnap~\cite{berg2014birdsnap}, Cars~\cite{yang2015large}, Aircraft~\cite{maji2013fine}, VOC2007~\cite{everingham2007pascal}, Pets~\cite{patino2016pets}.
For image classification tasks, we report the accuracy results in Table~\ref{tab:classification_performance}. Following the standard evaluation metrics in \cite{chen2020simple, khosla2020supervised}, we present mean-per-class accuracy for Caltech and Flowers, while reporting top-1 accuracy for all other datasets. 
The results indicate that CLOP generally outperforms competing methods. Although SimMatch and SimMatch-V2 achieve slightly higher accuracy on the Food and DTD datasets, the performance gain is less than $2\%$. In contrast, CLOP surpasses these methods on the remaining four datasets, with the most significant improvement exceeding $3.5\%$ on CIFAR-100. 
For object detection tasks, we report mean average precision (mAP) for VOC2007, mean-per-class accuracy for Aircraft and Pets, and top-1 accuracy for Birdsnap. The results, presented in Table~\ref{tab:object_detection_performance}, demonstrate that CLOP outperforms competing methods on three out of five datasets.

\subsection{Ablation Studies}
\label{subsec-ablation}
In this section, we present the experimental results for image classification, conducted with various batch sizes and learning rates on the CIFAR-100 and ImageNet datasets. For baseline methods, we implement the InfoNCE \cite{wu2018unsupervised} with a supervised linear classifier for semi-supervised learning and the SupCon \cite{khosla2020supervised} for fully-supervised learning. All experiments are performed using the SimCLR \cite{chen2020simple} framework with ResNet-50 (4x) \cite{he2016deep}.

\textbf{Effect of Prototype Initialization.} To assess the impact of prototype initialization in our framework, we conduct an ablation study comparing non-orthonormal prototypes with orthonormal prototypes on CIFAR-100 under full supervision. As shown in Table~\ref{tab:ablation_prototypes}, initializing prototypes with orthonormal vectors consistently yields better performance across all statistical measures. Specifically, orthonormal initialization improves the mean top-1 accuracy from 0.768 to 0.780 and reduces the standard deviation from 0.015 to 0.009, indicating both improved performance and training stability.

\textbf{Orthonormal (CLOP) vs ETF prototypes.} To ensure a fair comparison with ETF, we also evaluate performance using ETF prototypes as an alternative to the orthonormal prototypes. For fully-supervised learning, we utilize all labels in the training datasets for both SupCon and CLOP. 
In the semi-supervised setting, we employ 10\% of the labeled data for both linear classifier and CLOP training.
We report both top-1 classification accuracy on ImageNet in Figure \ref{fig:exp-lr-bs} and accuracies on CIFAR100 and Tiny-ImageNet in Appendix \ref{sec-app-exp} using the supervised linear classifier.

\begin{table}
    \centering
    \begin{tabular}{lcc}
        \toprule
        \textbf{Statistic} & \textbf{Non-Orthonormal Prototypes} & \textbf{Orthonormal Prototypes} \\
        \midrule
        Mean              & 0.768 & \textbf{0.780} \\
        Median            & 0.762 & \textbf{0.781} \\
        Std               & 0.015 & \textbf{0.009} \\
        Lower Quantile    & 0.760 & \textbf{0.775} \\
        Upper Quantile    & 0.775 & \textbf{0.786} \\
        \bottomrule
    \end{tabular}
    \caption{Ablation study on the effect of prototype initialization. We compare non-orthonormal and orthonormal initialization over 10 independent runs on CIFAR-100 under the full label setting. Results report the distribution statistics of Top-1 classification accuracy.}
    \label{tab:ablation_prototypes}
\end{table}

\begin{figure*}
    \centering
    \includegraphics[height=0.197\linewidth, clip, trim = {0cm 0cm 0cm 0cm}]{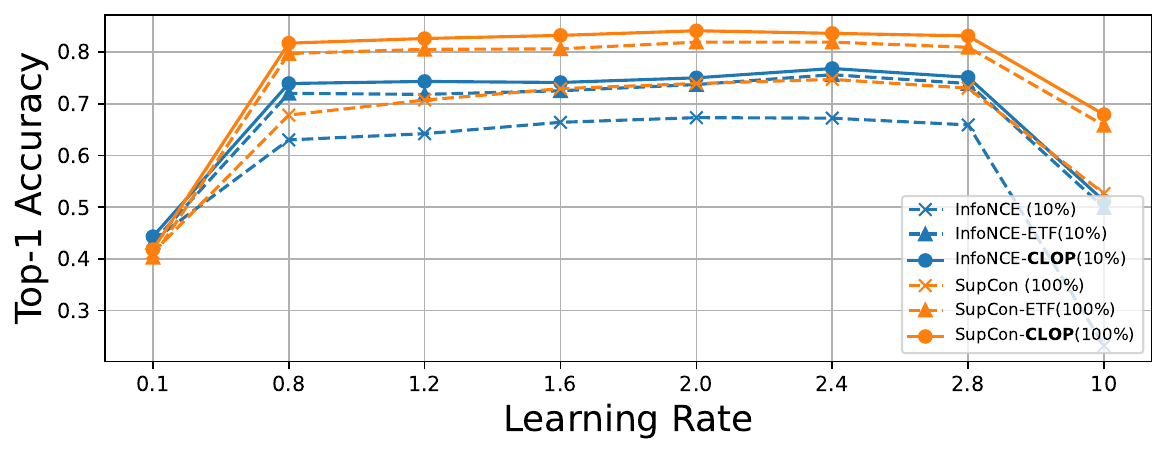}
    \includegraphics[height=0.197\linewidth, clip, trim = {1cm 0cm 0cm 0cm}]{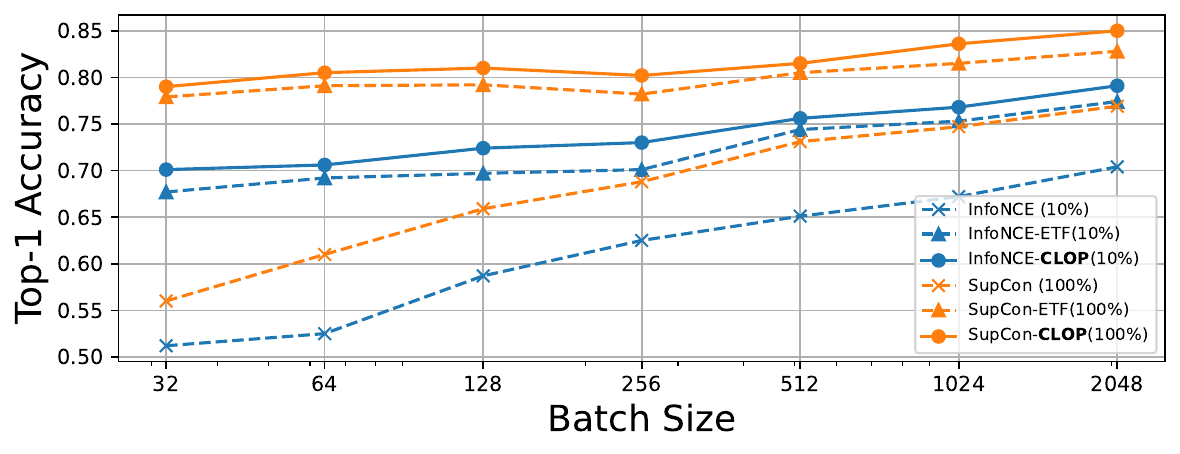}
    \caption{Top-1 classification accuracy on ImageNet across different learning rates and batch sizes. The percentage of labels used for supervised training is indicated in the legend.}
    \label{fig:exp-lr-bs}
\end{figure*}

\begin{table*}
\caption{Ablation studies on $\lambda$, similarity metric for CLOP, and augmentation strategies.}
\setlength{\tabcolsep}{5pt} 
\label{tab-ablation}
\vspace{-0.3cm}
\centering
\subtable[Acc of different $\lambda$.]{
\begin{tabular}{ccc}
\toprule
{$\lambda$} & {Top-1} & {Top-5}  \\ \midrule
0.1 & 0.745 & 0.935 \\ 
0.5 & 0.740 & 0.931 \\ 
1.0 & 0.754 & \textbf{0.938}\\ 
1.5 & \textbf{0.760} & 0.937 \\ \bottomrule
\end{tabular}
\label{ablation-lambda}
}
\subtable[Acc of different similarity metric.]{
\label{tab:similarity_ablation}
\begin{tabular}{lcc}
\toprule
{Similarity Metric} &{Top-1} &{Top-5}\\ \midrule
Cosine &\textbf{0.754} &\textbf{0.938}\\ 
Euclidean &0.749 &0.933\\ 
Manhattan &0.715 &0.899\\ \bottomrule
\end{tabular}
}
\subtable[Acc of augmentation strategies.]{
\begin{tabular}{lcc}
    \toprule
    Augmentation & Top-1 & Top-5 \\
    \midrule
    AutoAugment &0.625 &0.847 \\
    SimCLR &\textbf{0.754} &\textbf{0.938}\\
    RandAugment &0.726 &0.899 \\
    \bottomrule
\end{tabular}
\label{tab:aug_ablation}
}
\vspace{-0.2cm}
\end{table*}

\textbf{CLOP Prevents Collapse with Large Learning Rates.} We trained models with learning rates ranging from $0.1$ to $10$ on CIFAR-100 and ImageNet for 200 epochs and Tiny-ImageNet for 100 epochs, using a batch size of 1024. The corresponding classification accuracies are presented in Figure~\ref{fig:exp-lr-bs} and \ref{appfig:exp-lr}. Across both datasets, CLOP consistently outperforms the baseline methods. Moreover, as demonstrated by Section~\ref{sec-theory}, excessively large learning rates can lead to complete collapse, as clearly observed in the baseline methods at a learning rate of $10$ on both datasets. However, with the incorporation of CLOP into the loss function, we observe a significantly smaller performance degradation on both datasets.

\textbf{CLOP Enables Smaller Batch Sizes.} We trained models with batch sizes of 32, 64, 128, 256, 512, 1024, and 2048 on CIFAR-100 and ImageNet for 200 epochs and on Tiny-ImageNet for 100 epochs. The learning rate was fixed at $\left(0.3 \times \text{batch size} / 256\right)$ for optimal performance. The corresponding classification accuracies are presented in Figure~\ref{fig:exp-lr-bs} and \ref{appfig:exp-batchsize}. CLOP consistently outperformed the baseline methods across all batch sizes. 
As reported in the original papers \cite{chen2020simple,khosla2020supervised}, contrastive learning performs optimally when the batch size exceeds 1024, a finding corroborated by our experiments. However, with the addition of CLOP, we observe significantly less performance degradation at smaller batch sizes. Remarkably, CLOP achieved similar accuracy with a batch size of 32 compared to the baseline SupCon with a batch size of 2048 for CIFAR-100.

\textbf{Tuning $\lambda$.} To evaluate the sensitivity of the tuning parameter $\lambda$ in CLOP, we trained the model with SupCon loss across different $\lambda$ values, keeping the batch size fixed at 1024. The classification accuracy on both CIFAR-100 is reported in Table~\ref{ablation-lambda}. We observe that the performance remains stable for $\lambda$ values ranging from $0.1$ to $1.5$, with $\lambda = 1.0$ and $\lambda = 1.5$ yielding the best overall performance.

\textbf{Choice of Similarity Metric.} To evaluate the impact of different similarity functions on \eqref{CLOP_loss}, we trained the same ResNet-50 (4x) architecture on CIFAR-100 using cosine similarity, Euclidean similarity, and Manhattan similarity. The results, presented in Table \ref{tab:similarity_ablation}, indicate that cosine similarity, which aligns with $\mathcal{L}_{CL}$ in \eqref{CLOP_loss}, achieves the highest performance.

\textbf{Augmentation Strategies.}
To evaluate the impact of augmentation strategies on CLOP, we trained the same ResNet-50 (4x) model on CIFAR-100 with a batch size of 1024. We selected three commonly used augmentation methods:
1) RandAugment: Augmentation with three operations randomly chosen from all image processing functions in PyTorch (e.g., padding, resizing, cropping, rotation, color jitter, Gaussian blur, inversion, contrast adjustment, equalization); 2) AutoAugment using the ImageNet policy proposed in \cite{cubuk2018autoaugment}; 3) SimCLR Augmentation Policy from \cite{chen2020simple}. The results are shown in Table \ref{tab:aug_ablation}, indicate that SimCLR augmentation works best with CLOP.

\section{Conclusion and Limitations}
\label{sec-conclusion}
In this paper, we addressed the issue of dimensional collapse in contrastive learning by analyzing the impact of learning rates on representation stability. We identified a critical threshold beyond which standard contrastive objectives, particularly those based on cosine similarity, converge to degenerate low-rank solutions. Motivated by this observation, we proposed \textbf{CLOP}, a novel semi-supervised loss function that encourages embeddings to align with a set of orthonormal prototypes. By explicitly promoting the formation of class-specific orthogonal subspaces, CLOP enhances the expressiveness and separability of learned representations.
Empirical results across a range of classification and object detection benchmarks validate the effectiveness of CLOP. The method consistently outperforms strong baselines under both balanced and imbalanced label regimes, and demonstrates robustness across varying learning rates and batch sizes.

Despite these promising results, CLOP assumes a fixed number of well-separated classes and relies on static prototype initialization, which may limit its applicability in fine-grained or evolving-label scenarios. Future work could explore adaptive prototype mechanisms and extend the approach to tasks involving hierarchical or multi-label structures.

\bibliography{ref.bib}
\bibliographystyle{plain}  
\clearpage
\appendix

\section{Proof of Lemma~\ref{subspace-embeddings}}
\label{proof1}

\begin{proof}[Proof of Lemma \ref{subspace-embeddings}]
Consider \(\L_i\) as the i-th loss term of \(\L_{\text{InfoNCE}}\), defined by the following expression:
\begin{align*}
 \L_i := - \log  \mathbb{P}_i
\end{align*}
where $\mathbb{P}_i$ denotes the probability that i-th embedding choose its positive pair as closest neighbor:
\begin{align*}
    \mathbb{P}_i := \frac{\exp(\z_i^\top \z_{j(i)}/\tau)}{\exp(\z_i^\top \z_{j(i)}/\tau) + \sum_{a \notin \{i, j(i)\}} \exp(\z_i^\top \z_a/\tau)}
\end{align*}

As detailed in \cite{yeh2022decoupled}, the gradient of \(\L_i\) with respect to \(\z_i\), \(\z_{j(i)}\), and \(\z_a\) can be derived as follows:
\begin{align*}
 &-\frac{\partial \L_i}{\partial\z_i} := (1- \mathbb{P}_i)/\tau \left( \z_{j(i)} -\sum_{a \notin \{i, j(i)\}} \frac{\exp(\z_i^\top \z_a/\tau)}{\sum_{b \notin \{i, j(i)\}}  \exp(\z_i^\top \z_b/\tau)} \z_a \right) \\
 &-\frac{\partial \L_i}{\partial\z_{j(i)}} := \frac{(1- \mathbb{P}_i)}{\tau} \z_{i} \\
 &-\frac{\partial \L_i}{\partial\z_{a}} := -\frac{(1-\mathbb{P}_i)}{\tau} \frac{\exp(\z_i^\top \z_a/\tau)}{\sum_{b \notin \{i, j(i)\}}  \exp(\z_i^\top \z_b/\tau)}\z_{i} 
\end{align*}

In the standard setup of self-supervised learning, for any sample, there is one positive pair among \(I\) and the remainder are all negative pairs. By aggregating all the gradient respect to a single sample, we have the gradient of InfoNCE respect to $\z_i$:
\begin{align*}
 -\frac{\partial \L_{\text{InfoNCE}}}{\partial\z_i} := \frac{(1-\mathbb{P}_i)+ (1-\mathbb{P}_{j(i)})}{\tau}  \z_{j(i)} -&\sum_{a \notin \{i, j(i)\}} \frac{(1-\mathbb{P}_i)}{\tau} \frac{\exp(\z_i^\top \z_a/\tau)}{\sum_{b \notin \{i, j(i)\}}  \exp(\z_i^\top \z_b/\tau)} \z_a  \\
 & -\sum_{a \notin \{i, j(i)\}} \frac{(1-\mathbb{P}_a)}{\tau} \frac{\exp(\z_i^\top \z_a/\tau)}{\sum_{b \notin \{a, j(a)\}}  \exp(\z_a^\top \z_b/\tau)}\z_a 
\end{align*}

Now, considering the first scenario, where all embeddings equal, that means that \(\z_i = \z_{j(i)} = \z_a = \z^*\) for all \(a \in I\), the loss terms \(\mathbb{P}_i\), \(\mathbb{P}_{j(i)}\), and \(\mathbb{P}_a\) converge to a constant \(\mathbb{P}^*\), given by:
\begin{align*}
    \mathbb{P}_i = \mathbb{P}_{j(i)} = \mathbb{P}_a = -\log \frac{1}{|I| - 1} := \mathbb{P}^*
\end{align*}
Consequently, the gradient of \(\L_{\text{InfoNCE}}\) with respect to \(\z_i\) under this assumption reduces to zero, aligning with our expectations:
\begin{align*}
 -\frac{\partial \L_{\text{InfoNCE}}}{\partial \z_i} &= \frac{2(1-\mathbb{P}^*)}{\tau} \z^* - 2(|I|-2) \frac{(1-\mathbb{P}^*)}{\tau} \frac{1}{|I| - 2} \z^* = 0
 \end{align*}

We establish the existence of local minima in scenarios where all embeddings are identical. Now, we consider the second scenario where all embeddings generated reside within
 the same rank-1 subspace. Denoting \(\z^*\) as their unit basis, we can represent each embedding \(\z_i\) as:
\begin{equation*}
    \z_i = \alpha \z^*, \quad \alpha \in \{-1, 1\}, \quad \forall i
\end{equation*}
The gradient of the loss function \(\L_{\text{InfoNCE}}\) with respect to \(\z_i\) simplifies to:
\begin{align*}
     -\frac{\partial \L}{\partial \z_i} = \beta \z_i
\end{align*}
Here, \(\beta\) is a scalar that aggregates contributions from all relevant weights.

It is important to note that \(\z^i\) represents the normalized output of the function \(f\), with \(\mathbf{x}_i\) denoting the original, unnormalized embedding. This implies the following relation:
\begin{align*}
    -\frac{\partial \L}{\partial \mathbf{x}_i} &=  -\frac{\partial \L}{\partial \z_i} \frac{\partial \z_i}{\partial \mathbf{x}_i} = \frac{1}{\|\x_i\|_2}\left( \mathbb{I} - \frac{\x_i\x_i^\top}{\|\x_i\|_2^2} \right)\beta \z_i  = \frac{\beta}{\|\x_i\|_2}\left( \z_i - \frac{\x_i}{\|\x_i\|_2} \right)  = 0,
\end{align*}
where \(\mathbb{I}\) represents the identity matrix.

\end{proof}



\section{Gradient Analysis of Repulsive Force from InfoNCE Loss}
\label{proof2}
Recall that
\begin{equation*}
    \mathcal{L}_{\text{infoNCE}} = - \sum_{i \in I} \log \frac{\exp(\mathbf{z}_i^\top \mathbf{z}_{j(i)}/\tau)}{\sum_{a \neq i} \exp(\mathbf{z}_i^\top \mathbf{z}_a/\tau)} 
\end{equation*}

Assume that the model can successfully merge the positive pair into the same embedding, the loss of repulsive becomes:
\begin{equation*}
    \mathcal{L}_{\text{repulsive}} = - \sum_{i \in I} \log \frac{\exp(1/\tau)}{\sum_{a \neq i, j(i)} \exp(\mathbf{z}_i^\top \mathbf{z}_a/\tau) + \exp(1/\tau)} 
\end{equation*}

We start by rewriting the loss for a given sample i:
$$\ell_i = -\log \frac{\exp(1/\tau)}{\exp(1/\tau) + \sum_{a \in \mathcal{N}(i)} \exp\Bigl(\mathbf{z}_i^\top \mathbf{z}_a/\tau\Bigr)},$$
where we denote
$$\mathcal{N}(i) := \{\,a: a\neq i \text{ and } a\neq j(i)\,\}.$$

Define the denominator
$$D_i := \exp(1/\tau) + \sum_{a\in \mathcal{N}(i)} \exp\Bigl(\mathbf{z}_i^\top \mathbf{z}_a/\tau\Bigr).$$
Then
$$\ell_i = \log\Bigl[D_i\Bigr] - \frac{1}{\tau}.$$

Define $ \mathbb{P}_{ia}$ as the probability that sample $i$ choose sample a as closest neighbor,
$$ \mathbb{P}_{ia} := \frac{\exp\Bigl(\mathbf{z}_i^\top\mathbf{z}_a/\tau\Bigr)}{\exp(1/\tau) + \sum_{a\in \mathcal{N}(i)} \exp\Bigl(\mathbf{z}_i^\top \mathbf{z}_a/\tau\Bigr)} = \frac{\exp\Bigl(\mathbf{z}_i^\top\mathbf{z}_a/\tau\Bigr)}{D_i} $$

\textbf{Taking the gradient of $\ell_i$ with respect to $\mathbf{z}_i$,}
$$\frac{\partial}{\partial \mathbf{z}_i}\exp\Bigl(\mathbf{z}_i^\top\mathbf{z}_a/\tau\Bigr)
=\frac{1}{\tau}\exp\Bigl(\mathbf{z}_i^\top\mathbf{z}_a/\tau\Bigr)\mathbf{z}_a.$$
Thus, differentiating the log term gives:
$$\frac{\partial \ell_i}{\partial \mathbf{z}_i}
=\frac{1}{D_i}\sum_{a\in\mathcal{N}(i)} \frac{1}{\tau}\exp\Bigl(\mathbf{z}_i^\top\mathbf{z}_a/\tau\Bigr)\mathbf{z}_a.$$
Rewriting in terms of the softmax probabilities,
$$\frac{\partial \ell_i}{\partial \mathbf{z}_i}
=\frac{1}{\tau}\sum_{a\in\mathcal{N}(i)} \mathbb{P}_{ia}\,\mathbf{z}_a,$$

\textbf{Taking the gradient of $\ell_i$ with respect to $\mathbf{z}_a$,}
$$\frac{\partial}{\partial \mathbf{z}_a}\exp\Bigl(\mathbf{z}_i^\top\mathbf{z}_a/\tau\Bigr)
=\frac{1}{\tau}\exp\Bigl(\mathbf{z}_i^\top\mathbf{z}_a/\tau\Bigr)\mathbf{z}_i.$$
Thus, differentiating the log term gives:
$$\frac{\partial \ell_i}{\partial \mathbf{z}_a}
=\frac{1}{D_i}\frac{1}{\tau}\exp\Bigl(\mathbf{z}_i^\top\mathbf{z}_a/\tau\Bigr)\mathbf{z}_i.$$
Rewriting in terms of the softmax probabilities,
$$\frac{\partial \ell_i}{\partial \mathbf{z}_a}
=\frac{1}{\tau} \mathbb{P}_{ia}\,\mathbf{z}_i.$$

Therefore, the gradient of $\mathcal{L}_{\text{repulsive}}$ is:
$$\frac{\partial \mathcal{L}_{\text{repulsive}}}{\partial \mathbf{z}_i} = \frac{\partial \ell_i}{\partial \mathbf{z}_i} + \sum_{a \in \mathcal{N}(i)}\frac{\partial \ell_a}{\partial \mathbf{z}_i} = \frac{1}{\tau}\sum_{a\in\mathcal{N}(i)} \mathbb{P}_{ia}\,\mathbf{z}_a + \frac{1}{\tau}\sum_{a \in \mathcal{N}(i)}\mathbb{P}_{ai}\,\mathbf{z}_a =  \frac{1}{\tau}\sum_{a\in\mathcal{N}(i)} (\mathbb{P}_{ia} + \mathbb{P}_{ai})\,\mathbf{z}_a$$

Since the probability matrix $\mathbb{P}$ is obviously symmetric,
$$\frac{\partial \mathcal{L}_{\text{repulsive}}}{\partial \mathbf{z}_i} =  \frac{2}{\tau}\sum_{a\in\mathcal{N}(i)} \mathbb{P}_{ia}\mathbf{z}_a$$

Thus, for the gradient of the before the normalization,
\begin{align*}
    \frac{\partial \mathcal{L}}{\partial \mathbf{x}_i} &=  \frac{\partial \mathcal{L}}{\partial \mathbf{z}_i} \frac{\partial \mathbf{z}_i}{\partial \mathbf{x}_i} = \frac{1}{\|\mathbf{x}_i\|_2}\left( \mathbb{I} - \frac{\mathbf{x}_i(\mathbf{x}_i)^\top}{\|\mathbf{x}_i\|_2^2} \right)\frac{2}{\tau}\sum_{a\in\mathcal{N}(i)} \mathbb{P}_{ia}\mathbf{z}_a = \frac{1}{\|\mathbf{x}_i\|_2}\left( \mathbb{I} - {\mathbf{z}_i(\mathbf{z}_i)^\top} \right)\frac{2}{\tau}\sum_{a\in\mathcal{N}(i)} \mathbb{P}_{ia}\mathbf{z}_a 
\end{align*}

Given the first iteration of points as $\{\mathbf{x}^{0}_i\}$, then, after the first iteration of descent, the next embeddings are:
$$\mathbf{x}_i^1 = \mathbf{x}^{0}_i - \eta \frac{1}{\|\mathbf{x}_i^0\|_2}\left( \mathbb{I} - {\mathbf{z}^0_i(\mathbf{z}_i^0)^\top} \right)\frac{2}{\tau}\sum_{a\in\mathcal{N}(i)} \mathbb{P}_{ia}\,\mathbf{z}^0_a$$

Thus, the mean of $\mathbf{x}_i^1$ can be calculated as:
$$\bm{\mu}^1 = \frac{1}{N}\sum_{i=1}^N \mathbf{x}_i^1 = \frac{1}{N}\sum_{i=1}^N  \mathbf{x}^{0}_i - \frac{2\eta}{N\tau}\sum_{i=1}^N \sum_{a\in\mathcal{N}(i)} \frac{1}{\|\mathbf{x}_i^0\|_2}\mathbb{P}_{ia} \left( \mathbb{I} - {\mathbf{z}_i^0(\mathbf{z}_i^0)^\top} \right)\mathbf{z}^0_a$$


Since the negative pair relationship is symmetric, ie. $a \in \mathcal{N}(i) \Longleftrightarrow i \in \mathcal{N}(a)$, then
\begin{align*}
    \sum_{i=1}^N \sum_{a\in\mathcal{N}(i)} \mathbb{P}_{ia}\frac{1}{\|\mathbf{x}_i^0\|_2}\left( \mathbb{I} - {\mathbf{z}_i^0(\mathbf{z}_i^0)^\top} \right)\mathbf{z}^0_a =
    \sum_{a=1}^N \sum_{i\in\mathcal{N}(a)} \mathbb{P}_{ia}\frac{1}{\|\mathbf{x}_i^0\|_2}\left( \mathbb{I} - {\mathbf{z}_i^0(\mathbf{z}_i^0)^\top} \right)\mathbf{z}^0_a
\end{align*}

Thus,
\begin{align*}
    \bm{\mu}^1 &= \frac{1}{N}\sum_{i=1}^N  \mathbf{x}^{0}_i - \frac{2\eta}{N\tau}\sum_{a=1}^N \sum_{i\in\mathcal{N}(a)} \mathbb{P}_{ia}\frac{1}{\|\mathbf{x}_i^0\|_2}\left( \mathbb{I} - {\mathbf{z}_i^0(\mathbf{z}_i^0)^\top} \right)\mathbf{z}^0_a \\
    &= \frac{1}{N}\sum_{i=1}^N  \mathbf{x}^{0}_i - \frac{2\eta}{N\tau}\sum_{i=1}^N \sum_{a\in\mathcal{N}(i)} \frac{\mathbb{P}_{ia}}{\|\mathbf{x}_a^0\|_2}\left( \mathbb{I} - {\z_a^0(\z_a^0)^\top} \right)\mathbf{z}^0_i \\
    &= \frac{1}{N}\sum_{i=1}^N  \mathbf{x}^{0}_i - \frac{2\eta}{N\tau}\sum_{i=1}^N \sum_{a\in\mathcal{N}(i)} \frac{\mathbb{P}_{ia}}{\|\mathbf{x}_a^0\|_2}\mathbf{z}^0_i + \frac{2\eta}{N\tau}\sum_{i=1}^N \sum_{a\in\mathcal{N}(i)} \frac{\mathbb{P}_{ia}}{\|\mathbf{x}_a^0\|_2}\z_a^0(\z_a^0)^\top \mathbf{z}^0_i\\
    &= \frac{1}{N}\sum_{i=1}^N  \mathbf{x}^{0}_i - \frac{2\eta}{N\tau}\sum_{i=1}^N \sum_{a\in\mathcal{N}(i)} \frac{\mathbb{P}_{ia}}{\|\mathbf{x}_a^0\|_2}\mathbf{z}^0_i + \frac{2\eta}{N\tau}\sum_{i=1}^N \sum_{a\in\mathcal{N}(i)} \frac{\mathbb{P}_{ia}}{\|\mathbf{x}_i^0\|_2}\z_i^0(\z_i^0)^\top \mathbf{z}^0_a\\
    &= \frac{1}{N}\sum_{i=1}^N  \mathbf{x}^{0}_i - \frac{2\eta}{N\tau}\sum_{i=1}^N \sum_{a\in\mathcal{N}(i)} \mathbb{P}_{ia}\frac{1}{\|\mathbf{x}_a^0\|_2}\mathbf{z}^0_i + \frac{2\eta}{N\tau}\sum_{i=1}^N \sum_{a\in\mathcal{N}(i)} \mathbb{P}_{ia}\frac{(\z_i^0)^\top \mathbf{z}^0_a}{\|\mathbf{x}_i^0\|_2}\z_i^0\\    
    &= \frac{1}{N}\sum_{i=1}^N  \mathbf{x}^{0}_i - \frac{2\eta}{N\tau}\sum_{i=1}^N \sum_{a\in\mathcal{N}(i)} \mathbb{P}_{ia}\left(\frac{1}{\|\mathbf{x}_a^0\|_2} - \frac{(\z_i^0)^\top \mathbf{z}^0_a}{\|\mathbf{x}_i^0\|_2}\right)\frac{1}{\|\mathbf{x}^0_i\|_2}\mathbf{x}^0_i\\    
    &= \frac{1}{N}\sum_{i=1}^N  \left[ \left(1 - \frac{2\eta}{\tau}\right) \frac{1}{\|\mathbf{x}^0_i\|_2}\sum_{a\in\mathcal{N}(i)}  \mathbb{P}_{ia}\left(\frac{1}{\|\mathbf{x}_a^0\|_2} - \frac{(\z_i^0)^\top \mathbf{z}^0_a}{\|\mathbf{x}_i^0\|_2}\right)\right]\mathbf{x}^0_i\\ 
    &=  \left(1 - \frac{2\eta}{\tau}\right) \frac{1}{N}\sum_{i=1}^N  \left[\frac{1}{\|\mathbf{x}^0_i\|_2}\sum_{a\in\mathcal{N}(i)}  \mathbb{P}_{ia}\left(\frac{1}{\|\mathbf{x}_a^0\|_2} - \frac{(\z_i^0)^\top \mathbf{z}^0_a}{\|\mathbf{x}_i^0\|_2}\right)\right]\mathbf{x}^0_i\\
        \end{align*}
    \begin{align*}
\end{align*}

Since $-1 \leq \z_i^\top \z_a \leq 1$, Denote
\begin{align*}
    \beta_{ia} &:= \left[\frac{1}{\|\mathbf{x}^0_i\|_2}\sum_{a\in\mathcal{N}(i)}  \mathbb{P}_{ia}\left(\frac{1}{\|\mathbf{x}_a^0\|_2} - \frac{(\z_i^0)^\top \mathbf{z}^0_a}{\|\mathbf{x}_i^0\|_2}\right)\right]\\
    &\leq \frac{1}{\|\mathbf{x}^0_i\|_2}\sum_{a\in\mathcal{N}(i)}  \mathbb{P}_{ia}\left(\frac{1}{\|\mathbf{x}_a^0\|_2} + \frac{1}{\|\mathbf{x}_i^0\|_2}\right)\\
    &\leq \frac{2}{\min_j\|\mathbf{x}^0_j\|^2_2}\sum_{a\in\mathcal{N}(i)}  \mathbb{P}_{ia}\\
\end{align*}

And we can lowerbound sum of probability by
$$
   \sum_{a\in\mathcal{N}(i)} \mathbb{P}_{ia} 
   =
   \frac{\sum_{a\in\mathcal{N}(i)}\exp\Bigl((\mathbf{z}^{0}_i)^\top(\mathbf{z}_a^0)/\tau\Bigr)}{\exp(1/\tau) + \sum_{a\in \mathcal{N}(i)} \exp\Bigl((\mathbf{z}^{0}_i)^\top(\mathbf{z}_a^0)/\tau\Bigr)} \leq 
   \frac{{{|\mathcal{N}|}\exp(1/\tau)}}{\exp(1/\tau) + |\mathcal{N}| \exp(1/\tau)}  
   =
   \frac{{{|\mathcal{N}|}}}{1 + |\mathcal{N}| }  
$$

Thus, denote $\sigma := \min_j\|\mathbf{x}^0_j\|_2$, we show that
$$\beta_{ia} \leq \frac{2}{\sigma^2} \frac{{{|\mathcal{N}|}}}{1 + |\mathcal{N}| }  $$

Therefore,
$$\|\bm{\mu}_1\|_2 \leq \left| \left(1 - \frac{2\eta}{\tau}\right)\frac{2}{\sigma^2} \frac{{{|\mathcal{N}|}}}{1 + |\mathcal{N}| }\right| \|\bm{\mu}_0\|_2$$

If we want to limit the increase of mean, we want the coefficient to be less than 1, which means that

$$-\frac{\sigma^2 (1 + |\mathcal{N}|)}{2|\mathcal{N}|} < 1 - \frac{2\eta}{\tau} < \frac{\sigma^2 (1 + |\mathcal{N}|)}{2|\mathcal{N}|}$$
$$-1 - \frac{\sigma^2 (1 + |\mathcal{N}|)}{2|\mathcal{N}|} < - \frac{2\eta}{\tau} < -1 + \frac{\sigma^2 (1 + |\mathcal{N}|)}{2|\mathcal{N}|}$$
$$1 + \frac{\sigma^2 (1 + |\mathcal{N}|)}{2|\mathcal{N}|} > \frac{2\eta}{\tau} > 1 - \frac{\sigma^2 (1 + |\mathcal{N}|)}{2|\mathcal{N}|}$$
$$\frac{\tau}{2} \left(1 + \frac{\sigma^2 (1 + |\mathcal{N}|)}{2|\mathcal{N}|} \right) > \eta > \frac{\tau}{2} \left(1 - \frac{\sigma^2 (1 + |\mathcal{N}|)}{2|\mathcal{N}|} \right)$$

\section{Extra Experiments on CLOP}
\label{sec-app-exp}

\begin{figure}
    \centering
    \includegraphics[height=0.2\linewidth,clip, trim = {0 1cm 0cm 0cm}]{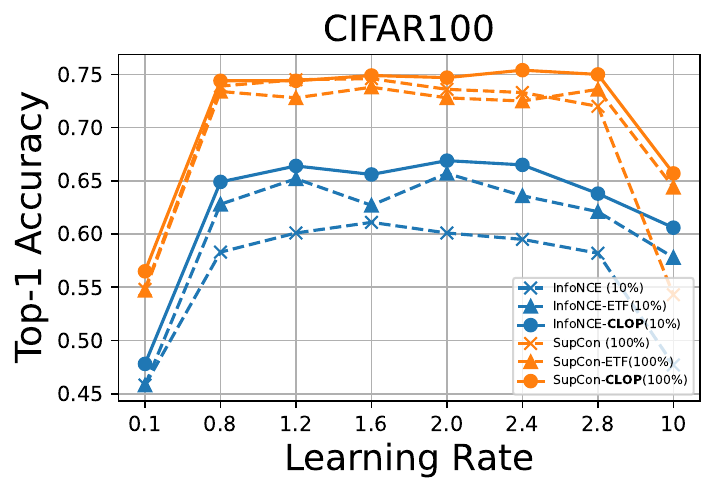}
    \includegraphics[height=0.2\linewidth, clip, trim = {1cm 1cm 0cm 0cm}]{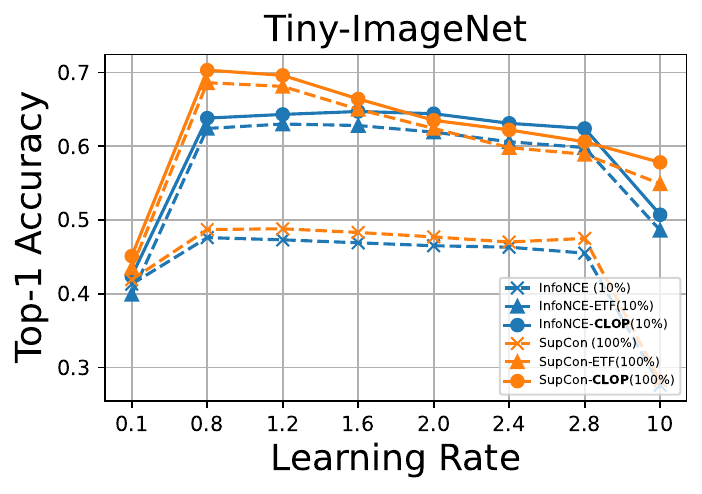}
    \includegraphics[height=0.2\linewidth, clip, trim = {1cm 1cm 0cm 0cm}]{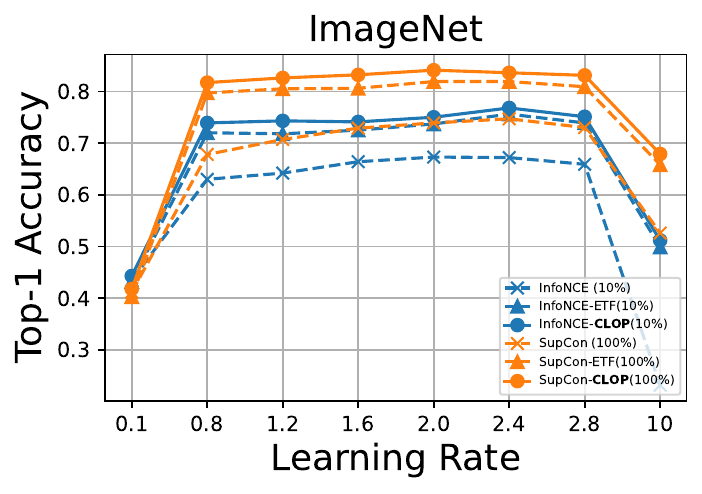}
    \includegraphics[height=0.2\linewidth, clip, trim = {0cm 0cm 0cm 0.85cm}]{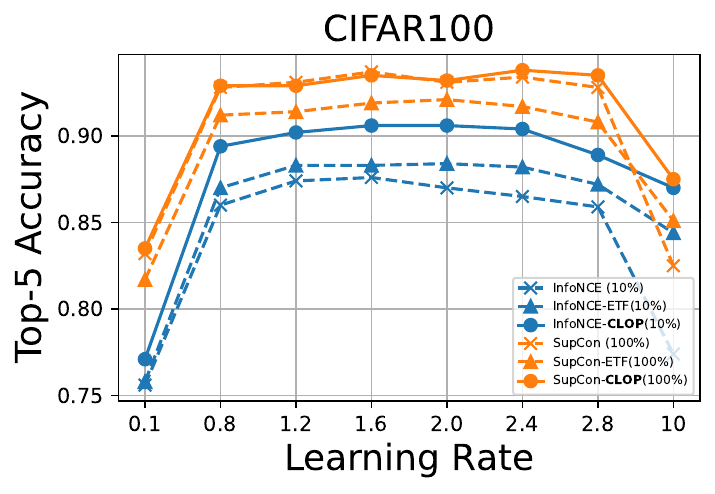}
    \includegraphics[height=0.2\linewidth, clip, trim = {1cm 0cm 0cm 0.85cm}]{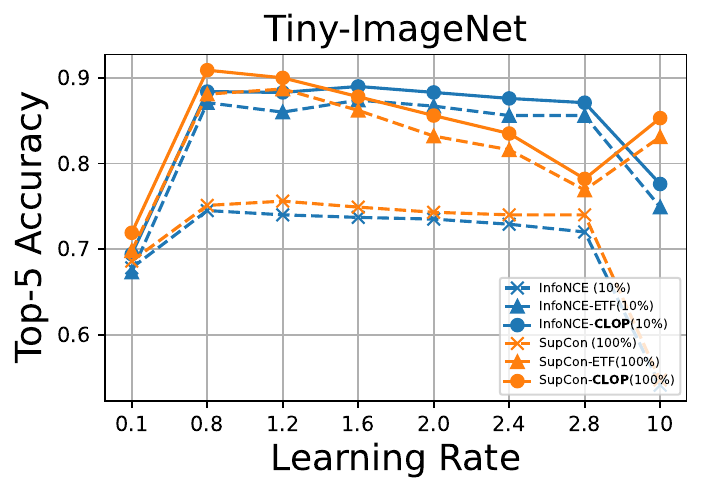}
    \includegraphics[height=0.2\linewidth, clip, trim = {1cm 0cm 0cm 0.85cm}]{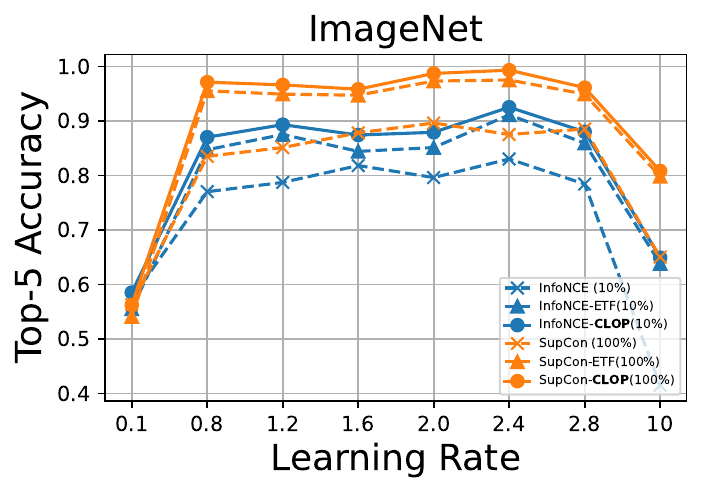}
    \caption{Top-1 classification accuracy across different learning rates. The percentage of labels used for supervised training is indicated in the legend.}
    \vspace{-0.2cm}
    \label{appfig:exp-lr}
\end{figure}

\setlength{\columnsep}{10pt} 
\begin{figure}
    \centering
    \includegraphics[height=0.2\linewidth,clip, trim = {0 1cm 0cm 0cm}]{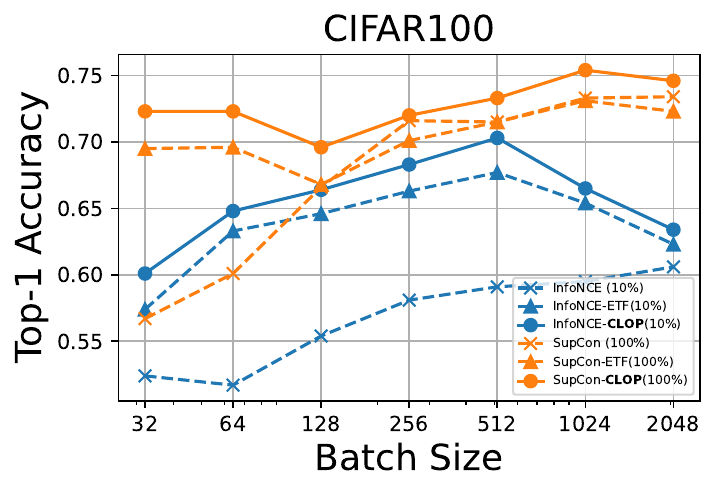}
    \includegraphics[height=0.2\linewidth, clip, trim = {1cm 1cm 0cm 0cm}]{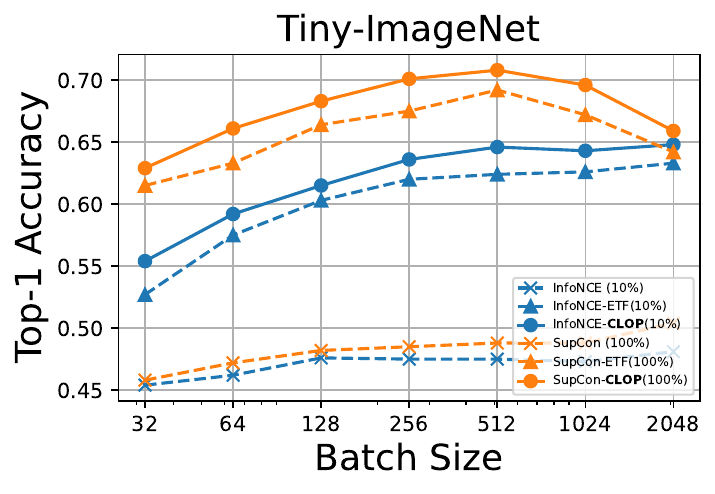}
    \includegraphics[height=0.2\linewidth, clip, trim = {1cm 1cm 0cm 0cm}]{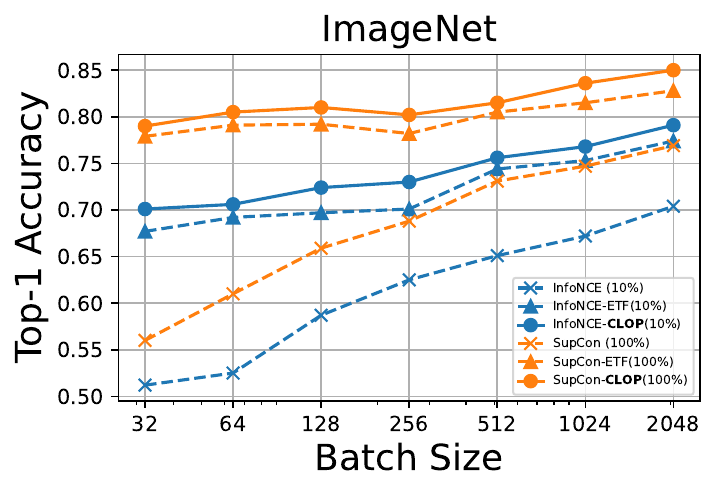}
    \includegraphics[height=0.2\linewidth, clip, trim = {0cm 0cm 0cm 0.85cm}]{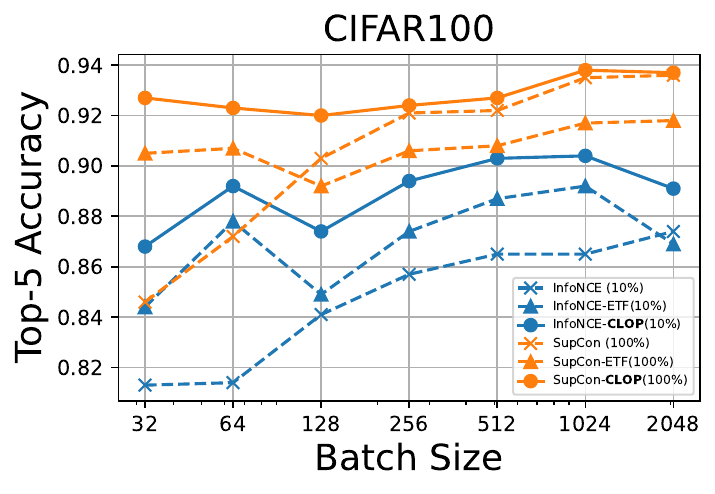}
    \includegraphics[height=0.2\linewidth, clip, trim = {1cm 0cm 0cm 0.85cm}]{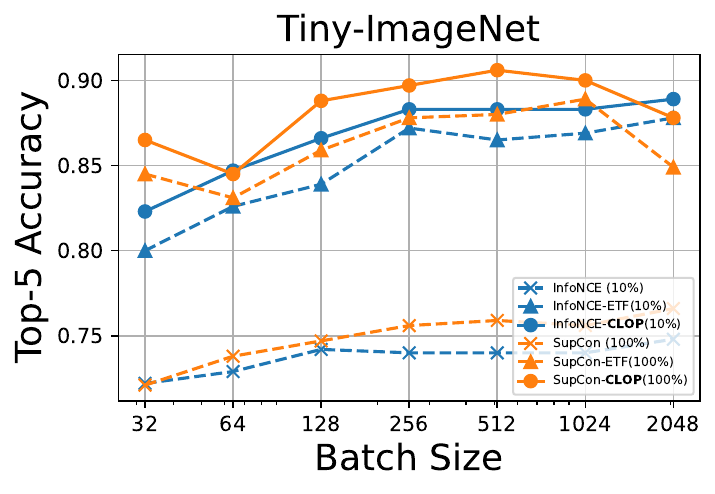}
     \includegraphics[height=0.2\linewidth, clip, trim = {1cm 0cm 0cm 0.85cm}]{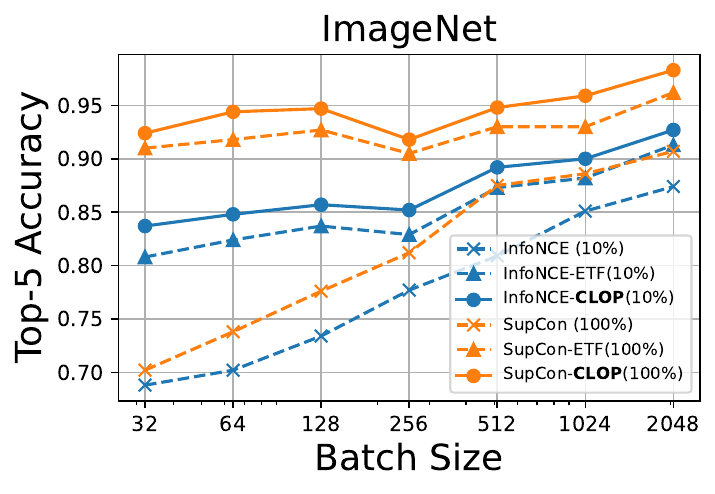}
    \vspace{-0.3cm}
    \caption{Top-1 classification accuracy across different batch sizes. The percentage of labels used for supervised training is indicated in the legend.}
    \vspace{-0.2cm}
    \label{appfig:exp-batchsize}
\end{figure}

\clearpage

\section*{NeurIPS Paper Checklist}

\begin{enumerate}

\item {\bf Claims}
    \item[] Question: Do the main claims made in the abstract and introduction accurately reflect the paper's contributions and scope?
    \item[] Answer: \answerYes{} 
    \item[] Justification: Refer to Abstract.
    \item[] Guidelines:
    \begin{itemize}
        \item The answer NA means that the abstract and introduction do not include the claims made in the paper.
        \item The abstract and/or introduction should clearly state the claims made, including the contributions made in the paper and important assumptions and limitations. A No or NA answer to this question will not be perceived well by the reviewers. 
        \item The claims made should match theoretical and experimental results, and reflect how much the results can be expected to generalize to other settings. 
        \item It is fine to include aspirational goals as motivation as long as it is clear that these goals are not attained by the paper. 
    \end{itemize}

\item {\bf Limitations}
    \item[] Question: Does the paper discuss the limitations of the work performed by the authors?
    \item[] Answer: \answerYes{} 
    \item[] Justification: Refer to Section~\ref{sec-conclusion}.
    \item[] Guidelines:
    \begin{itemize}
        \item The answer NA means that the paper has no limitation while the answer No means that the paper has limitations, but those are not discussed in the paper. 
        \item The authors are encouraged to create a separate "Limitations" section in their paper.
        \item The paper should point out any strong assumptions and how robust the results are to violations of these assumptions (e.g., independence assumptions, noiseless settings, model well-specification, asymptotic approximations only holding locally). The authors should reflect on how these assumptions might be violated in practice and what the implications would be.
        \item The authors should reflect on the scope of the claims made, e.g., if the approach was only tested on a few datasets or with a few runs. In general, empirical results often depend on implicit assumptions, which should be articulated.
        \item The authors should reflect on the factors that influence the performance of the approach. For example, a facial recognition algorithm may perform poorly when image resolution is low or images are taken in low lighting. Or a speech-to-text system might not be used reliably to provide closed captions for online lectures because it fails to handle technical jargon.
        \item The authors should discuss the computational efficiency of the proposed algorithms and how they scale with dataset size.
        \item If applicable, the authors should discuss possible limitations of their approach to address problems of privacy and fairness.
        \item While the authors might fear that complete honesty about limitations might be used by reviewers as grounds for rejection, a worse outcome might be that reviewers discover limitations that aren't acknowledged in the paper. The authors should use their best judgment and recognize that individual actions in favor of transparency play an important role in developing norms that preserve the integrity of the community. Reviewers will be specifically instructed to not penalize honesty concerning limitations.
    \end{itemize}

\item {\bf Theory assumptions and proofs}
    \item[] Question: For each theoretical result, does the paper provide the full set of assumptions and a complete (and correct) proof?
    \item[] Answer: \answerYes{} 
    \item[] Justification: Refer to Section~\ref{sec-theory}, where the range of the learning rate is derived under the assumption that positive sample pairs are successfully merged into the same embedding. Proof can be found in Appendix~\ref{proof1} and Appendix~\ref{proof2}.
    
    \item[] Guidelines:
    \begin{itemize}
        \item The answer NA means that the paper does not include theoretical results. 
        \item All the theorems, formulas, and proofs in the paper should be numbered and cross-referenced.
        \item All assumptions should be clearly stated or referenced in the statement of any theorems.
        \item The proofs can either appear in the main paper or the supplemental material, but if they appear in the supplemental material, the authors are encouraged to provide a short proof sketch to provide intuition. 
        \item Inversely, any informal proof provided in the core of the paper should be complemented by formal proofs provided in appendix or supplemental material.
        \item Theorems and Lemmas that the proof relies upon should be properly referenced. 
    \end{itemize}

    \item {\bf Experimental result reproducibility}
    \item[] Question: Does the paper fully disclose all the information needed to reproduce the main experimental results of the paper to the extent that it affects the main claims and/or conclusions of the paper (regardless of whether the code and data are provided or not)?
    \item[] Answer: \answerYes{} 
    \item[] Justification: The formulation is explained in Section~\ref{sec-model}, and the model's backbone structure is explained in Section~\ref{sec-experiment}.
    \item[] Guidelines:
    \begin{itemize}
        \item The answer NA means that the paper does not include experiments.
        \item If the paper includes experiments, a No answer to this question will not be perceived well by the reviewers: Making the paper reproducible is important, regardless of whether the code and data are provided or not.
        \item If the contribution is a dataset and/or model, the authors should describe the steps taken to make their results reproducible or verifiable. 
        \item Depending on the contribution, reproducibility can be accomplished in various ways. For example, if the contribution is a novel architecture, describing the architecture fully might suffice, or if the contribution is a specific model and empirical evaluation, it may be necessary to either make it possible for others to replicate the model with the same dataset, or provide access to the model. In general. releasing code and data is often one good way to accomplish this, but reproducibility can also be provided via detailed instructions for how to replicate the results, access to a hosted model (e.g., in the case of a large language model), releasing of a model checkpoint, or other means that are appropriate to the research performed.
        \item While NeurIPS does not require releasing code, the conference does require all submissions to provide some reasonable avenue for reproducibility, which may depend on the nature of the contribution. For example
        \begin{enumerate}
            \item If the contribution is primarily a new algorithm, the paper should make it clear how to reproduce that algorithm.
            \item If the contribution is primarily a new model architecture, the paper should describe the architecture clearly and fully.
            \item If the contribution is a new model (e.g., a large language model), then there should either be a way to access this model for reproducing the results or a way to reproduce the model (e.g., with an open-source dataset or instructions for how to construct the dataset).
            \item We recognize that reproducibility may be tricky in some cases, in which case authors are welcome to describe the particular way they provide for reproducibility. In the case of closed-source models, it may be that access to the model is limited in some way (e.g., to registered users), but it should be possible for other researchers to have some path to reproducing or verifying the results.
        \end{enumerate}
    \end{itemize}

\item {\bf Open access to data and code}
    \item[] Question: Does the paper provide open access to the data and code, with sufficient instructions to faithfully reproduce the main experimental results, as described in supplemental material?
    \item[] Answer: \answerYes{} 
    \item[] Justification: Refer the the supplemental codes and readme file.
    \item[] Guidelines:
    \begin{itemize}
        \item The answer NA means that paper does not include experiments requiring code.
        \item Please see the NeurIPS code and data submission guidelines (\url{https://nips.cc/public/guides/CodeSubmissionPolicy}) for more details.
        \item While we encourage the release of code and data, we understand that this might not be possible, so “No” is an acceptable answer. Papers cannot be rejected simply for not including code, unless this is central to the contribution (e.g., for a new open-source benchmark).
        \item The instructions should contain the exact command and environment needed to run to reproduce the results. See the NeurIPS code and data submission guidelines (\url{https://nips.cc/public/guides/CodeSubmissionPolicy}) for more details.
        \item The authors should provide instructions on data access and preparation, including how to access the raw data, preprocessed data, intermediate data, and generated data, etc.
        \item The authors should provide scripts to reproduce all experimental results for the new proposed method and baselines. If only a subset of experiments are reproducible, they should state which ones are omitted from the script and why.
        \item At submission time, to preserve anonymity, the authors should release anonymized versions (if applicable).
        \item Providing as much information as possible in supplemental material (appended to the paper) is recommended, but including URLs to data and code is permitted.
    \end{itemize}

\item {\bf Experimental setting/details}
    \item[] Question: Does the paper specify all the training and test details (e.g., data splits, hyperparameters, how they were chosen, type of optimizer, etc.) necessary to understand the results?
    \item[] Answer: \answerYes{} 
    \item[] Justification: Refer to Section~\ref{subsec-ablation} for ablation studies on the choice of model configuration.
    \item[] Guidelines:
    \begin{itemize}
        \item The answer NA means that the paper does not include experiments.
        \item The experimental setting should be presented in the core of the paper to a level of detail that is necessary to appreciate the results and make sense of them.
        \item The full details can be provided either with the code, in appendix, or as supplemental material.
    \end{itemize}

\item {\bf Experiment statistical significance}
    \item[] Question: Does the paper report error bars suitably and correctly defined or other appropriate information about the statistical significance of the experiments?
    \item[] Answer: \answerYes{} 
    \item[] Justification: Refer to Section~\ref{sec-experiment}, Appendix~\ref{sec-app-exp}.
    \item[] Guidelines:
    \begin{itemize}
        \item The answer NA means that the paper does not include experiments.
        \item The authors should answer "Yes" if the results are accompanied by error bars, confidence intervals, or statistical significance tests, at least for the experiments that support the main claims of the paper.
        \item The factors of variability that the error bars are capturing should be clearly stated (for example, train/test split, initialization, random drawing of some parameter, or overall run with given experimental conditions).
        \item The method for calculating the error bars should be explained (closed form formula, call to a library function, bootstrap, etc.)
        \item The assumptions made should be given (e.g., Normally distributed errors).
        \item It should be clear whether the error bar is the standard deviation or the standard error of the mean.
        \item It is OK to report 1-sigma error bars, but one should state it. The authors should preferably report a 2-sigma error bar than state that they have a 96\% CI, if the hypothesis of Normality of errors is not verified.
        \item For asymmetric distributions, the authors should be careful not to show in tables or figures symmetric error bars that would yield results that are out of range (e.g. negative error rates).
        \item If error bars are reported in tables or plots, The authors should explain in the text how they were calculated and reference the corresponding figures or tables in the text.
    \end{itemize}

\item {\bf Experiments compute resources}
    \item[] Question: For each experiment, does the paper provide sufficient information on the computer resources (type of compute workers, memory, time of execution) needed to reproduce the experiments?
    \item[] Answer: \answerYes{} 
    \item[] Justification: Please refer to the last paragraph of Section~\ref{sec-model}. This paper adopts the SimCLR framework with minor modifications that do not introduce significant computational overhead or affect memory usage and execution time; see \cite{chen2020simple} for further information.
    \item[] Guidelines:
    \begin{itemize}
        \item The answer NA means that the paper does not include experiments.
        \item The paper should indicate the type of compute workers CPU or GPU, internal cluster, or cloud provider, including relevant memory and storage.
        \item The paper should provide the amount of compute required for each of the individual experimental runs as well as estimate the total compute. 
        \item The paper should disclose whether the full research project required more compute than the experiments reported in the paper (e.g., preliminary or failed experiments that didn't make it into the paper). 
    \end{itemize}
    
\item {\bf Code of ethics}
    \item[] Question: Does the research conducted in the paper conform, in every respect, with the NeurIPS Code of Ethics \url{https://neurips.cc/public/EthicsGuidelines}?
    \item[] Answer: \answerYes{} 
    \item[] Justification: This research is conducted following the code of Ethics.
    \item[] Guidelines:
    \begin{itemize}
        \item The answer NA means that the authors have not reviewed the NeurIPS Code of Ethics.
        \item If the authors answer No, they should explain the special circumstances that require a deviation from the Code of Ethics.
        \item The authors should make sure to preserve anonymity (e.g., if there is a special consideration due to laws or regulations in their jurisdiction).
    \end{itemize}

\item {\bf Broader impacts}
    \item[] Question: Does the paper discuss both potential positive societal impacts and negative societal impacts of the work performed?
    \item[] Answer: \answerNA{} 
    \item[] Justification: This paper aims to improve current semi-supervised learning accuracy in vision-related tasks. The authors do not anticipate any direct societal impacts.
    \item[] Guidelines:
    \begin{itemize}
        \item The answer NA means that there is no societal impact of the work performed.
        \item If the authors answer NA or No, they should explain why their work has no societal impact or why the paper does not address societal impact.
        \item Examples of negative societal impacts include potential malicious or unintended uses (e.g., disinformation, generating fake profiles, surveillance), fairness considerations (e.g., deployment of technologies that could make decisions that unfairly impact specific groups), privacy considerations, and security considerations.
        \item The conference expects that many papers will be foundational research and not tied to particular applications, let alone deployments. However, if there is a direct path to any negative applications, the authors should point it out. For example, it is legitimate to point out that an improvement in the quality of generative models could be used to generate deepfakes for disinformation. On the other hand, it is not needed to point out that a generic algorithm for optimizing neural networks could enable people to train models that generate Deepfakes faster.
        \item The authors should consider possible harms that could arise when the technology is being used as intended and functioning correctly, harms that could arise when the technology is being used as intended but gives incorrect results, and harms following from (intentional or unintentional) misuse of the technology.
        \item If there are negative societal impacts, the authors could also discuss possible mitigation strategies (e.g., gated release of models, providing defenses in addition to attacks, mechanisms for monitoring misuse, mechanisms to monitor how a system learns from feedback over time, improving the efficiency and accessibility of ML).
    \end{itemize}
    
\item {\bf Safeguards}
    \item[] Question: Does the paper describe safeguards that have been put in place for responsible release of data or models that have a high risk for misuse (e.g., pretrained language models, image generators, or scraped datasets)?
    \item[] Answer: \answerNA{} 
    \item[] Justification: The release of the model does not have a high risk of misuse. 
    \item[] Guidelines:
    \begin{itemize}
        \item The answer NA means that the paper poses no such risks.
        \item Released models that have a high risk for misuse or dual-use should be released with necessary safeguards to allow for controlled use of the model, for example by requiring that users adhere to usage guidelines or restrictions to access the model or implementing safety filters. 
        \item Datasets that have been scraped from the Internet could pose safety risks. The authors should describe how they avoided releasing unsafe images.
        \item We recognize that providing effective safeguards is challenging, and many papers do not require this, but we encourage authors to take this into account and make a best faith effort.
    \end{itemize}

\item {\bf Licenses for existing assets}
    \item[] Question: Are the creators or original owners of assets (e.g., code, data, models), used in the paper, properly credited and are the license and terms of use explicitly mentioned and properly respected?
    \item[] Answer: \answerYes{} 
    \item[] Justification: Refer to Section~\ref{sec-experiment}.
    \item[] Guidelines:
    \begin{itemize}
        \item The answer NA means that the paper does not use existing assets.
        \item The authors should cite the original paper that produced the code package or dataset.
        \item The authors should state which version of the asset is used and, if possible, include a URL.
        \item The name of the license (e.g., CC-BY 4.0) should be included for each asset.
        \item For scraped data from a particular source (e.g., website), the copyright and terms of service of that source should be provided.
        \item If assets are released, the license, copyright information, and terms of use in the package should be provided. For popular datasets, \url{paperswithcode.com/datasets} has curated licenses for some datasets. Their licensing guide can help determine the license of a dataset.
        \item For existing datasets that are re-packaged, both the original license and the license of the derived asset (if it has changed) should be provided.
        \item If this information is not available online, the authors are encouraged to reach out to the asset's creators.
    \end{itemize}

\item {\bf New assets}
    \item[] Question: Are new assets introduced in the paper well documented and is the documentation provided alongside the assets?
    \item[] Answer: \answerYes{} 
    \item[] Justification: The released code is well documented.
    \item[] Guidelines:
    \begin{itemize}
        \item The answer NA means that the paper does not release new assets.
        \item Researchers should communicate the details of the dataset/code/model as part of their submissions via structured templates. This includes details about training, license, limitations, etc. 
        \item The paper should discuss whether and how consent was obtained from people whose asset is used.
        \item At submission time, remember to anonymize your assets (if applicable). You can either create an anonymized URL or include an anonymized zip file.
    \end{itemize}

\item {\bf Crowdsourcing and research with human subjects}
    \item[] Question: For crowdsourcing experiments and research with human subjects, does the paper include the full text of instructions given to participants and screenshots, if applicable, as well as details about answerNA (if any)? 
    \item[] Answer: \answerNA{} 
    \item[] Justification: No crowdsourcing experiment is involved.
    \item[] Guidelines:
    \begin{itemize}
        \item The answer NA means that the paper does not involve crowdsourcing nor research with human subjects.
        \item Including this information in the supplemental material is fine, but if the main contribution of the paper involves human subjects, then as much detail as possible should be included in the main paper. 
        \item According to the NeurIPS Code of Ethics, workers involved in data collection, curation, or other labor should be paid at least the minimum wage in the country of the data collector. 
    \end{itemize}

\item {\bf Institutional review board (IRB) approvals or equivalent for research with human subjects}
    \item[] Question: Does the paper describe potential risks incurred by study participants, whether such risks were disclosed to the subjects, and whether Institutional Review Board (IRB) approvals (or an equivalent approval/review based on the requirements of your country or institution) were obtained?
    \item[] Answer: \answerNA{} 
    \item[] Justification: No human subjects are involved.
    \item[] Guidelines:
    \begin{itemize}
        \item The answer NA means that the paper does not involve crowdsourcing nor research with human subjects.
        \item Depending on the country in which research is conducted, IRB approval (or equivalent) may be required for any human subjects research. If you obtained IRB approval, you should clearly state this in the paper. 
        \item We recognize that the procedures for this may vary significantly between institutions and locations, and we expect authors to adhere to the NeurIPS Code of Ethics and the guidelines for their institution. 
        \item For initial submissions, do not include any information that would break anonymity (if applicable), such as the institution conducting the review.
    \end{itemize}

\item {\bf Declaration of LLM usage}
    \item[] Question: Does the paper describe the usage of LLMs if it is an important, original, or non-standard component of the core methods in this research? Note that if the LLM is used only for writing, editing, or formatting purposes and does not impact the core methodology, scientific rigorousness, or originality of the research, declaration is not required.
    \item[] Answer: \answerNo{} 
    \item[] Justification: No important, original, or non-standard component of the core methods in this research is involved with LLM usage.
    \item[] Guidelines:
    \begin{itemize}
        \item The answer NA means that the core method development in this research does not involve LLMs as any important, original, or non-standard components.
        \item Please refer to our LLM policy (\url{https://neurips.cc/Conferences/2025/LLM}) for what should or should not be described.
    \end{itemize}

\end{enumerate}


\end{document}